%% file: paper.tex
\documentclass[runningheads]{llncs}
\usepackage[T1]{fontenc}
\usepackage{graphicx}
\usepackage{booktabs}
\usepackage[misc]{ifsym}
\newcommand{\corr}{(\Letter)}
\usepackage{algorithm}
\usepackage{algpseudocode}
\usepackage{amsfonts,amsmath}
\usepackage{dsfont}
\usepackage{hyperref}
\usepackage{multirow}
\usepackage{subcaption}
\usepackage{svg}
\usepackage{tabularx}
\usepackage{tikz}

\usetikzlibrary{arrows.meta,positioning}
\newcolumntype{C}{>{\centering\arraybackslash}X}

\newcommand*{\D}{\ensuremath{\mathcal{D}}}
\newcommand*{\Dsyn}{\ensuremath{\tilde{\mathcal{D}}}}
\newcommand*{\Dtst}{\ensuremath{\D_{\textnormal{test}}}}
\newcommand{\law}[1]{\ensuremath{\mathbb{#1}}}
\newcommand{\lawdat}{\ensuremath{\law{D}}}
\newcommand{\lawsyn}{\ensuremath{\tilde{\law{D}}}}

\newcommand*{\Prob}{\ensuremath{\mathbb{P}}}

\newcommand{\domain}[1]{\ensuremath{\mathcal{#1}}}
\newcommand{\ind}[1]{\ensuremath{\mathds{1}\{#1\}}}
\newcommand{\indep}{\ensuremath{\mathrel{\perp\!\!\!\perp}}}

\begin{document}

\title{Disparate Impact in Synthetic Data Generation}

\author{
  Paul Andrey\inst{1} \corr
  \and Batiste Le Bars\inst{1}
  \and Michaël Perrot\inst{1}
  \and Marc Tommasi\inst{1}
}
\authorrunning{P. Andrey et al.}
\institute{
  Univ. Lille, Inria, CNRS, Centrale Lille,
  UMR 9189 - CRIStAL, F-59000 Lille, France
  \email{\{paul.andrey,batiste.le-bars,michael.perrot,marc.tommasi\}@inria.fr}
}

\maketitle

\begin{abstract}
  We revisit the fairness notion of disparate impact for synthetic data
  generation (SDG), that assesses whether the utility of generated records
  is the same across sensitive groups.
  Our approach departs from existing work on fair SDG, that address the problem
  of correcting for undue biases in the observed distribution, hence redefining
  SDG as learning a distribution that is not that of the real data.
  By contrast, non-disparate impact is notably achieved when the synthetic
  and real distributions are the same.
  We expose reasons why SDG may fail to reach that solution and discuss why
  approximation and estimation errors occur and can be disparate across groups.
  We notably look into the expressive power of SDG methods relative to
  distribution complexity, sampling errors due to group proportions, and
  estimation errors induced by differential privacy mechanisms.
  We illustrate cases of disparate impact on both artificial and real-world
  data, focusing on SDG methods that rely on probabilistic graphical models.
  We also introduce a strategy of learning group-wise SDG models, that
  constitutes a simple baseline for analyzing and in some settings mitigating
  cases of disparate impact.

\end{abstract}

\newcommand\blfootnote[1]{%
  \begingroup
  \renewcommand\thefootnote{}\footnotetext{#1}%
  \endgroup
}
\blfootnote{
  \scriptsize
  Accepted at ECML PKDD 2026 (Research Track).\\
  First published in Lecture Notes in Computer Science (LNAI), volume 16946,
  pp 95-113, in 2026 by Springer Nature.
  Version of Record: \url{https://doi.org/10.1007/978-3-032-37673-2\_6}\\
  This expanded version includes appendices that detail experimental results,
  and a minor rewrite of definition~\ref{def:disparate-impact}.
}

\section{Introduction}

\subsection{Context and Motivation}

Synthetic Data Generation (SDG) consists in generating a dataset of artificial
records that is statistically close to a real-world observed dataset. It may
notably be used to provide illustrative data from a domain for educational
purposes, or to enable learning statistics or even training machine learning
(ML) models without releasing private individual records~\cite{misc:synthea}.
Most state-of-the-art SDG methods provide formal privacy guarantees based on
the Differential Privacy (DP) framework~\cite{dp:dwork-roth}, as a protection
against privacy attacks that may otherwise reveal information about
the real-world training data~\cite{mia:stadler,mia:tapas,mia:tamis}.

In addition to privacy, a topic of great concern for trustworthy ML is
whether models are fair, meaning that their usage does not unduly
put some individuals or groups to a disadvantage.
The notion of algorithmic fairness is multi-faceted, due to the variety of
views on what constitutes an inequity, resulting in many distinct (at times
incompatible) fairness definitions co-existing in the
literature~\cite{fairness:barocas}.
In the context of SDG, fairness has mostly been regarded as a question of
dataset or distribution bias, where SDG acts as a pre-processing step
to learn fair classifiers~\cite{fairness:fairgan,fairness:cfgan,%
fairness:decaf,fairness:abroshan,fairness:prefair}.
This line of work assumes that the real-world data is biased, and redefines
the objective of SDG as learning a synthetic data distribution that is
structurally different from the observed one, in the sense that it enforces
some (conditional) independence between attributes that were not independent
in the real-world data. This approach relates to the notion
of fair representations~\cite{fairness:fair-representations}.

In this paper, we propose an alternative definition of fairness for SDG,
based on the notion of disparate impact~\cite{fairness:disparate-impact}.
We embrace the widely-studied framework of group
fairness~\cite{fairness:barocas}, where samples belong to mutually-exclusive
groups, and assess whether the utility of synthetic data relative to the
real-world one is equal across these groups.
This notion applies to virtually any SDG context, provided a utility metric
is defined, which encompasses but is not limited to the use of synthetic data
for ML.
In contrast with the definition of fair SDG as a de-biasing problem,
our approach assumes soundness of the real-world data and adopts the
standard SDG objective of estimating its underlying distribution.
As that distribution may be imperfectly estimated (whether due to the
limited expressive power of the SDG method, to the sampling error due to
observing a finite-size dataset, or to DP-induced estimation errors), we argue
that it is critical to assess whether these errors are disparate across groups.

Our goal is therefore to understand and illustrate mechanisms
responsible for disparate impact in SDG.
To the best of our knowledge, this is an under-explored research topic,
with existing works having solely focused on the disparate impact of
DP~\cite{eval:ganev2022,eval:bullwinkel2022}.
It may however be highly-relevant in real-world applications,
as generating synthetic samples that mis-represent certain groups may
hinder the way these groups are treated when that data is disseminated
either as illustration for human practitioners (e.g. for medical training) or
as training data for ML.

\subsection{Contributions}

We formalize and motivate the definition of fairness for SDG
as disparate impact. We then conduct experiments on both artificial and
real-world data to illustrate cases of disparate impact and understand their
causes.
We focus on SDG methods that rely on probabilistic graphical models
(PGM)~\cite{misc:graphical-models}, meaning that they learn a distribution
with an explicit conditional independence structure, represented as a graph.
Those methods~\cite{sdg:privbayes,sdg:mst,sdg:aim} have parametrizable
constraints as to which distributions they can learn, enabling us to question
the relative roles of approximation and estimation errors on disparate impact.

We notably investigate the relative and cumulative effects of the disparity
of sample size and the disparity of true distribution complexity across groups
on the resulting utility of synthetic data.
We also question the disparate impact of DP, that induces some randomness and
estimation errors to SDG, which may be more detrimental for some groups than
others~\cite{eval:ganev2022}, either structurally (for instance due to
disparate noise-to-signal ratio across varying-size groups) or stochastically.
Finally, we introduce a simple group-wise SDG modeling approach, and assess
how it affects both the overall utility and its disparity across groups.

\section{Disparate Impact of Synthetic Data Generation}\label{sec:contrib}

Given an empirical dataset $\D$ with underlying (unknown) distribution
$\lawdat$, SDG consists in producing a synthetic dataset $\Dsyn$, typically
by learning a generative distribution $\lawsyn$ and sampling from it.
In our experiments and part of our discussion, we specifically focus on SDG
methods that rely on a PGM~\cite{sdg:privbayes,sdg:mst,sdg:aim}. These methods
select a graph over attributes that defines a conditional independence
structure, and measure marginal or conditional statistics over $\D$ based
on that graph that define a joint distribution $\lawsyn$.
In order to provide DP guarantees, PGM-based methods add randomness to
the graph selection process and random noise to the measured statistics.
For completeness, we recall the formal definition of $(\epsilon,\delta)$-DP in
Definition~\ref{def:differential-privacy}.

\begin{definition}[Differential Privacy~\cite{dp:dwork-roth}]
  \label{def:differential-privacy}
  A randomized algorithm $\mathcal{A}$ with output space $\mathcal{O}$
  is $(\epsilon,\delta)$-DP for $\epsilon > 0$ and $\delta \in [0, 1[$
  if $\forall S \subseteq O,\ \forall D,\ D' = D \cup \{x\}$
  \begin{equation*}
    \Prob(\mathcal{A}(D) \in S) \leq
      e^\epsilon \Prob(\mathcal{A}(D') \in S) + \delta
  \end{equation*}
\end{definition}

We adopt the widely-studied framework of group fairness~\cite{fairness:barocas}
where samples belong to mutually-exclusive sensitive groups.
We consider $\lawdat$ to be a distribution over $(X, S)$ where
$X = (X_1, \dots, X_d)$ are $d$ feature attributes that can be in any domain
(although in our experiments they are categorical),
and $S = (S_1, \dots, S_{d_s})$ are $d_s$ categorical sensitive attributes,
the intersection of which define $K$ unique sensitive groups.
In the remainder of this paper, to keep notations light, we denote sensitive
groups as $S \in \domain{S} := \{1, \dots, K\}$, which somewhat conceals that
our definitions are valid for intersectional fairness settings.

We note individual records in lower case, and write conditional feature
distributions as $\lawdat_k := \lawdat(X | S = k)$
(respectively $\lawsyn_k := \lawsyn(X | S = k)$)
and group-wise feature datasets as $\D_k := \{x_i \in \D | s_i = k\}$
(respectively $\Dsyn_k := \{x_i \in \Dsyn | s_i = k\}$)
for any $k \in \domain{S}$.
In some contexts, we will consider an additional attribute $Y$ that is a
target variable for downstream ML tasks. In such cases, $Y$ is included in
both population-wide and group-wise distributions and datasets.

We use the term \emph{hypothesis class} to refer to the set of all possible
$\lawsyn$ that a given SDG method can output. When comparing distributions, we
also use the informal notion of \emph{complexity} to refer to how broad an
hypothesis class must be to contain them.
For PGM-based methods, hypothesis classes are defined by graph connectivity
properties, such as the maximum in-degree of graphs it can select. A higher
in-degree enables modeling more complex relationships between attributes,
hence when we state that a distribution has higher complexity than another
we typically mean that it requires a graph with higher in-degree for
its conditional independence structure to be accurately described.

\subsection{Unfairness as Disparate Impact}

We introduce a fairness definition for SDG that is based on the notion of
disparate impact~\cite{fairness:barocas}. It consists in assessing whether
synthetic data satisfy similar utility across sensitive groups, as formalized
in Definition~\ref{def:disparate-impact}.

\begin{definition}[Disparate Impact of SDG]\label{def:disparate-impact}
Given a post-processing function $f: \D \rightarrow F$,
a utility metric $u: F \times \D \rightarrow \mathbb{R}$,
a real-world testing dataset $\D' \sim \lawdat$
and a tolerance level $\tau \geq 0$,
a SDG method $\mathcal{A}$ has no $(f,u,\tau)$-disparate impact when
$\Dsyn = \mathcal{A}(\D)$ satisfies
\begin{equation}
    \label{eq:utility-parity}
    \forall k, s \in \domain{S},\
      \big|[u(f(\Dsyn), \D'_k) - u(f(\D), \D'_k)]
    - [u(f(\Dsyn), \D'_s) - u(f(\D), \D'_s)]\big| \leq \tau
\end{equation}
\end{definition}

Practical choices of $(f, u)$ are expected to be context-dependent, as is the
tolerance level $\tau$. Strict fairness is defined by setting $\tau=0$.

A natural instantiation of Definition~\ref{def:disparate-impact}
that is agnostic to downstream
tasks is to assess whether goodness-of-fit is similar across groups by taking
the identity as $f$ and a distance over probability distributions as $u$, such
as the Kullback–Leibler divergence or the Wasserstein distance $\mathcal{W}_p$.
In the latter case, when $\D' = \D$, fairness is achieved when
\begin{equation}
    \label{eq:distance-parity}
    \forall k, s \in \domain{S},\
    \mathcal{W}_p(\Dsyn_k, \D_k) = \mathcal{W}_p(\Dsyn_s, \D_s)
\end{equation}

Alternatively, $f$ may be a learning algorithm producing a classifier, and $u$
a standard evaluation metric such as its accuracy. In that case, disparate
impact assesses whether some groups suffer from a higher utility degradation
than others when using synthetic rather than real-world training data.
Note that SDG not having disparate impact does not imply that the learned
classifiers have similar accuracy across groups.
If one wants to learn a model that satisfies another fairness criterion, such
as accuracy or demographic parity, they should use a suitable downstream
learning algorithm.

\subsection{Sources of Disparate Impact}

By Definition~\ref{def:disparate-impact}, if $\lawsyn = \lawdat$ then the SDG
has non-disparate impact in expectation.
Therefore, our work applies to the (common) case where
the SDG method results in an imperfect estimation of $\lawdat$, and questions
whether modeling errors are equally distributed across sensitive groups.
We posit that these errors may be decomposed between approximation and
estimation errors.

On the one hand, approximation errors arise when $\lawdat$ is not part of the
hypothesis class of the SDG method, and can therefore not be learned.
If $\lawdat_1, \dots, \lawdat_K$ have heterogeneous complexity, the SDG method
may suffer an approximation error for some groups but not for others, resulting
in disparate impact.

On the other hand, estimation errors arise from observing a finite number of
samples $\D$ rather than $\lawdat$, so that statistics estimated from $\D$
to define $\lawsyn$ may differ from their true underlying value.
All else equal, a statistic $\mu(\D_k)$ is likely to have a higher estimation
error than its counterpart $\mu(\D_s)$ if $\mathbb{V}[X|S=k] > \mathbb{V}[X|S=s]$
(higher sample variance), or if $|\D_k| < |\D_s|$ (lower sample size).

Additionally, SDG methods may rely on heuristics to estimate the best candidate
for $\lawsyn$ in their hypothesis class that do not have strong formal
guarantees. It is notably the case for some PGM-based
methods~\cite{sdg:privbayes,sdg:aim}, that rely on greedy algorithms.
These heuristics are error-prone, and may have disparate impact.

Finally, DP mechanisms add randomness to SDG methods, which is bound to induce
some estimation errors. In general, one cannot predict whether DP-induced noise
will increase or decrease the disparity of errors across sensitive groups.
However, heterogeneous group proportions may cause the signal-to-noise ratio
of DP mechanisms to vary across groups, which is likely to advantage majority
groups, but can have opposite effects depending on SDG
methods~\cite{eval:ganev2022}.

\section{Related Work}\label{sec:related-work}

\subsubsection{Fair SDG and Generative Modeling}

To the best of our knowledge, all papers that address fairness for SDG~\cite{%
fairness:fairgan,fairness:cfgan,fairness:decaf,fairness:abroshan,%
fairness:prefair} formulate a de-biasing problem, that is related
to fair representations learning~\cite{fairness:fair-representations}.
Our work departs from this approach by formulating fairness for SDG as
disparate impact.
Some papers on image generation~\cite{fairness:choi2020,fairness:measuring-gen}
formulate fairness as ensuring that models produce samples for minority groups.
This problem is trivial in our setting, as we consider SDG methods that
explicitly model sensitive attributes, hence enabling conditional sampling.
Apart from SDG, \cite{fairness:gm-estimation} addresses fairness for graphical
model estimation, defined as learning a population-wide model that has
non-disparate impact on the utility of group-wise ones. This neglects the
potentially-disparate modeling errors of group-wise models, which our work
addresses.

\subsubsection{Evaluating SDG Methods}

Many works have conducted empirical assessments of the utility of SDG methods,
such as \cite{eval:tao2022,eval:ganev2024,eval:chen2025}.
They report and analyze population-wide utility, leaving disparate impact
unaddressed.
\cite{eval:cheng2021,eval:pereira2024} evaluate the fairness of downstream
classifiers but do so regardless of whether the real data would result in fair
models.
\cite{fairness:feedback-loops} demonstrated that unfairness can be induced by
SDG even when starting from non-biased data, which motivates our work, although
it considers a non-DP, multi-cycle setting that differs from ours.
The closest works to ours we know of are~\cite{eval:ganev2022} and~\cite{%
eval:bullwinkel2022}, which respectively showed that DP can have disparate
impact on the utility of synthetic data for under-represented groups, and on
the fairness of downstream classifiers.
Previous works outside of SDG have also studied the disparate impact of DP
on utility and identified it as a fairness issue~\cite{%
fairness:bagdasaryan-disparate-impact,fairness:pujol-decision-making}.
We go beyond these papers by investigating sources of disparate
impact in SDG beyond DP, and by assessing utility at the distribution level
in addition to measuring that of downstream ML classifiers.

\section{Experimental Setup}\label{sec:setup}

We conduct experiments that assess the disparate impact of a variety of
PGM-based SDG methods.
We first consider artificial settings where we control the properties of
true distributions, so as to investigate possible sources of disparate impact.
We then consider real-world socio-demographic data, to demonstrate real cases
of disparate impact and assess whether phenomenons assessed in our controlled
settings are at work.
In this section, we present the SDG methods that we compare, the datasets
that we use and the utility metrics that are reported in the upcoming
results section. Further details as to our methodology are provided in
Appendix~\ref{appendix:setup}. Our research is fully-reproducible, as
we make our source code, data and random seeds available online\footnote{
  \url{https://gitlab.inria.fr/magnet/thesepaulandrey/disparate-impact-sdg}
}.

\subsection{Compared SDG Methods}

We consider SDG methods that rely on a PGM and have DP guarantees.
These methods are state-of-the-art~\cite{eval:ganev2024,eval:chen2025},
and have explicit, configurable hypothesis classes, enabling us to know
whether $\lawdat$ is learnable in controlled settings, and to assess the effect
of the expressive power of methods on their utility and disparate impact.
MST~\cite{sdg:mst} uses a tree graph.
PrivBayes~\cite{sdg:privbayes} uses a bayesian network, meaning an acyclic
directed graph, constrained to limit the joint domain size of an attribute
and its parents to a quantity that is proportional to the DP budget $\epsilon$.
AIM~\cite{sdg:aim} uses a Markov Random Field, meaning an undirected graph,
constrained to have a junction tree representation of limited size and learned
from a parametrizable workload of marginals.

For each of these methods, we consider various privacy regimes:
$\epsilon \in \{1, 10, 100\}$. The smaller $\epsilon$, the stronger
the privacy.
For PrivBayes and MST, we use default hyper-parameters taken from the papers
they were respectively introduced in.
For AIM, we use all cliques of at most 3 attributes as workload, and limit
the computational cost of the method by adjusting other hyper-parameters,
detailed in Appendix~\ref{appendix:setup-sdg}.

We also consider non-DP methods.
We designed custom non-DP variants of MST and AIM, that have the same
hypothesis class as the base methods. Details about these are provided
in Appendix~\ref{appendix:setup-nondp}.
We also used the non-DP GreedyBayes~\cite{sdg:privbayes} algorithm, that is
mostly similar to PrivBayes, but is constrained by the maximum in-degree
$\rho$ of the graph. In practice, we use $\rho \in \{2, 3, 4\}$.

\begin{algorithm}[h]
\caption{Group-wise Synthetic Data Generation with DP}\label{alg:groupwise}
\begin{algorithmic}
\State \textbf{Inputs}:
  data $\D$,
  DP budget $(\epsilon, \delta)$,
  SDG method $\mathcal{A}: \D \times (\epsilon, \delta) \rightarrow \lawsyn$,
  $\alpha \in ]0, 1[$,
  $n \in \mathbb{N}$
\State \textbf{Output}: synthetic data $\Dsyn$
\State $\forall k \in \{1,\dots,S\},\
  c_k \gets \sum_{s \in \D} \ind{s = k} + \textnormal{Lap}(1 / (\alpha\epsilon))$
  \Comment{Measure noisy group sizes}
\State $(\tilde{p}_1, \dots, \tilde{p}_K) \gets (c_1, \dots, c_K) / \sum_{k=1}^K c_k$
  \Comment{Compute noisy group proportions}
\State $(n_1, \dots, n_K) \gets M(n;p_1,\dots,p_K)$ \Comment{Draw sample sizes from multinomial law}
\For{$k \in \{1,\dots,S\}$}
  \State $\lawsyn_k \gets \mathcal{A}(\D_k, ((1-\alpha) \epsilon, \delta))$
    \Comment{Learn group-wise distribution}
  \State $\Dsyn_k \gets \{x_i \sim \lawsyn_k\}_{i=1}^{n_k}$
    \Comment{Sample $n_k$ records}
  \State $S_k \gets \{k\}_{i=1}^{n_k}$
\EndFor
\State $\Dsyn \gets \textnormal{Concat}((\Dsyn_1, S_1), \dots, (\Dsyn_K, S_K))$
  \Comment{Assemble output $\Dsyn$}\\
\Return $\Dsyn$
\end{algorithmic}
\end{algorithm}

\subsection{Group-wise Synthetic Data Generation}

We introduce a meta-algorithm for SDG that consists in learning group-wise
distributions $\lawsyn_k$ via a given SDG method, and estimating group
proportions $\tilde{P}(S)$, so that synthetic records may be sampled as
$\tilde{s} \sim \tilde{P}(S)$ and $\tilde{x} \sim \lawsyn_{\tilde{s}}$.
We formalize this approach in Algorithm~\ref{alg:groupwise}, that applies to
any SDG method with DP guarantees and, as stated in Theorem~\ref{theorem:gw-dp},
achieves the same privacy guarantees.
Its non-DP counterpart has the same structure, but uses empirical sensitive
group proportions without noise.

In practice, we use $\alpha = 0.01$ in our experiments, allocating most of
the DP budget for SDG. Indeed, as group counts have a sensitivity of 1,
group proportions are well-preserved even for low values of $\alpha\epsilon$.

\begin{theorem}\label{theorem:gw-dp}
Algorithm~\ref{alg:groupwise} is $(\epsilon, \delta)$-DP with respect to $\D$.
\end{theorem}

\begin{proof}
The proof relies on standard results for DP~\cite{dp:dwork-roth}.
Noisy sensitive group proportions satisfy $\alpha \epsilon$-DP
thanks to the use of the Laplace mechanism.
The set of group-wise distributions learned on disjoint subsets of $\D$
with an $((1 - \alpha)\epsilon, \delta)$-DP algorithm is also
$((1 - \alpha)\epsilon, \delta)$-DP thanks to parallel composition.
Combining previous results provides $(\epsilon, \delta)$-DP guarantees thanks
to sequential composition.
Sampling preserves these guarantees thanks to the post-processing theorem.
\end{proof}

For the SDG methods we consider, this group-wise modeling strategy has a
broader hypothesis class than the wrapped method, hence reducing the potential
for approximation errors.
It notably excludes the possibility to select graphs in which
$\lawsyn(X) \indep S$, but leaves the possibility to model independence
when estimated conditional probabilities are the same across groups.
For PrivBayes and GreedyBayes, it also enables selecting group-wise graphs
that could not be fused into a valid bayesian network and are therefore
outside the hypothesis class of the base method.

On the other hand,
learning group-wise models results in a loss of statistical power when
estimating statistics that do not depend on $S$. It may therefore hinder
the utility of synthetic data. This impact may be disparate across groups,
notably when they have different proportions. Indeed, estimating the same
statistic over various $\D_k$ will result in distinct values either due to
sampling errors (proportional to the group size) or to DP-induced noise
(centered and stochastic, but with a noise-to-signal ratio that also
depends on group sizes).

Consequently, group-wise modeling is not expected to mitigate all cases of
disparate impact, but rather to introduce a trade-off between estimation and
approximation errors.
Thus we introduce it as a baseline to estimate this trade-off
empirically, and to mitigate disparate impact in settings where approximation
and algorithm-based estimation errors dominate sample-size-based ones.

\subsection{Controlled Distributions}

In order to illustrate and study structural causes of utility disparity,
we craft some distributions with controlled properties that we suspect
may cause SDG methods to have disparate impact.

We first introduce a family of distributions by taking an arbitrary
bayesian network graphical structure over 6 binary attributes, including a
sensitive attribute as root node.
Figures~\ref{fig:toy-graph:base-1} and~\ref{fig:toy-graph:base-0} present the
minimal graphs to represent $\lawdat_1$ and $\lawdat_0$ respectively.
The former is also the minimal graph to represent $\lawdat$.
It has a maximum in-degree of 2 (meaning that attributes are conditionally
dependent on at most 2 other attributes), and is thus learnable by all
considered SDG methods except MST. Group-wise conditional graphs have the same
maximum in-degree, but for $S = 1$ there is one more edge, which we expect to
put it at a disadvantage when using MST. We fix $P(S = 1) = 0.5$.

We define the previous setting as the \textit{base} one, and derive 3
additional settings by altering its distribution.
In the \textit{fewer-samples} setting, we set $P(S = 1) = 0.1$, in order to
study how group proportions can cause estimation errors.
In the \textit{higher-complexity} setting, illustrated by figure~\ref{%
fig:toy-graph:high}, we add an edge to the graph that results in a structural
difference of complexity between conditional distributions: in this setting,
$\lawdat$ and $\lawdat_1$ have a maximum in-degree of 3, whereas $\lawdat_0$
has a maximum in-degree of 2. GreedyBayes$(\rho=2)$ is therefore bound to
suffer an approximation error for $\lawdat$ and $\lawdat_1$.
In the \textit{double-disadvantage} setting, we combine modifications from both
previous ones, so that $\lawdat_1$ has both higher complexity and fewer samples
than $\lawdat_0$.

We instantiate these settings by drawing random conditional probabilities that
define specific distributions $\lawdat$. We draw 10 base distributions and
apply setting-defining modifications on them to ensure comparability across
settings. We sample 40 000 records per resulting distribution, on which
SDG is performed.

\tikzstyle{node} = [
  circle, draw, text width=1em, text centered, minimum height=1em
]
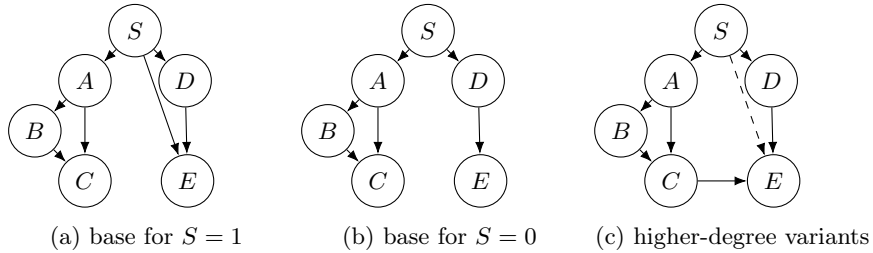
\begin{figure}[t]
  \centering
  \begin{subfigure}[t]{0.3\textwidth}
    \begin{tikzpicture}[node distance=0.5em and 0.5em]
      \node (S) at (0,0) [node]  {$S$};
      \node (A) [node, below left=of S] {$A$};
      \node (B) [node, below left=of A] {$B$};
      \node (C) [node, below right=of B] {$C$};
      \node (D) [node, below right=of S] {$D$};
      \node (E) [node, right=2em of C] {$E$};
      \draw[-{Latex[length=0.5em]}] (S) -- (A);
      \draw[-{Latex[length=0.5em]}] (A) -- (B);
      \draw[-{Latex[length=0.5em]}] (A) -- (C);
      \draw[-{Latex[length=0.5em]}] (B) -- (C);
      \draw[-{Latex[length=0.5em]}] (S) -- (D);
      \draw[-{Latex[length=0.5em]}] (D) -- (E);
      \draw[-{Latex[length=0.5em]}] (S) -- (E);  
    \end{tikzpicture}
    \caption{base for $S=1$}
    \label{fig:toy-graph:base-1}
  \end{subfigure}%
  ~
  \begin{subfigure}[t]{0.3\textwidth}
    \begin{tikzpicture}[node distance=0.5em and 0.5em]
      \node (S) at (0,0) [node]  {$S$};
      \node (A) [node, below left=of S] {$A$};
      \node (B) [node, below left=of A] {$B$};
      \node (C) [node, below right=of B] {$C$};
      \node (D) [node, below right=of S] {$D$};
      \node (E) [node, right=2em of C] {$E$};
      \draw[-{Latex[length=0.5em]}] (S) -- (A);
      \draw[-{Latex[length=0.5em]}] (A) -- (B);
      \draw[-{Latex[length=0.5em]}] (A) -- (C);
      \draw[-{Latex[length=0.5em]}] (B) -- (C);
      \draw[-{Latex[length=0.5em]}] (S) -- (D);
      \draw[-{Latex[length=0.5em]}] (D) -- (E);
    \end{tikzpicture}
    \caption{base for $S=0$}
    \label{fig:toy-graph:base-0}
  \end{subfigure}%
  ~
  \begin{subfigure}[t]{0.3\textwidth}
    \begin{tikzpicture}[node distance=0.5em and 0.5em]
      \node (S) at (0,0) [node]  {$S$};
      \node (A) [node, below left=of S] {$A$};
      \node (B) [node, below left=of A] {$B$};
      \node (C) [node, below right=of B] {$C$};
      \node (D) [node, below right=of S] {$D$};
      \node (E) [node, right=2em of C] {$E$};
      \draw[-{Latex[length=0.5em]}] (S) -- (A);
      \draw[-{Latex[length=0.5em]}] (A) -- (B);
      \draw[-{Latex[length=0.5em]}] (A) -- (C);
      \draw[-{Latex[length=0.5em]}] (B) -- (C);
      \draw[-{Latex[length=0.5em]}] (S) -- (D);
      \draw[-{Latex[length=0.5em]}] (D) -- (E);
      \draw[-{Latex[length=0.5em]},dashed] (S) -- (E);  
      \draw[-{Latex[length=0.5em]}] (C) -- (E);  
    \end{tikzpicture}
    \caption{higher-degree variants}
    \label{fig:toy-graph:high}
  \end{subfigure}%
  \caption{Graphical structure of controlled distributions}
  \label{fig:toy-graph}
\end{figure}

\subsection{American Community Survey Data}

In order to illustrate a real-world case of SDG unfairness, we run experiments
on American Community Survey (ACS) data, acquired using
the folktables library~\cite{misc:folktables} and pre-processed to define an
income-level-prediction binary task.

We restrict the dataset to records from the five most-populated States
(California, Texas, Florida, New-York and Pennsylvania), and from the two
majority ethnic groups: caucasian (white) and african-american (black),
making up for nearly 490 000 records.
We define sensitive attributes as gender (male / female) and ethnicity
(white / black), resulting in four sensitive groups that have heterogeneous
proportions: gender groups are close to being balanced, but there are about
ten times more white than black people in this dataset.

We randomly split the ACS dataset into three subsets, respectively denoted
as training (80~\% of records), validation (10~\%) and testing (10~\%) sets.
Details on how they are used to train downstream classifiers are provided in
Appendix~\ref{appendix:setup-acs}.
To ensure robustness, we draw 10 such splits, conduct our
experiments on each, and compute the
average and standard deviation of results across replicas.
For each replica, SDG methods produce 500 000 records based on the training set.

\subsection{Utility metrics}

We consider two types of utility metrics to assess the quality of the
synthetic datasets and their disparity across sensitive groups.

For both controlled and ACS data, we compute distances between empirical
probability distributions, in the vein of Equation~\ref{eq:distance-parity}.
In controlled settings, we report the total variation distance (TVD), defined as
$\textnormal{TVD}(P, Q) =
  \frac{1}{2} \sum_{\omega \in \Omega} |P(\omega) - Q(\omega)|$
for probability mass functions $P$ and $Q$ over the event set $\Omega$.
For ACS data, TVD is somewhat unpractical because the total domain size of
joint distributions is very high. We therefore focus our analysis on the
average of TVD values for all $n$-way marginals, as was notably done
in~\cite{eval:chen2025}. We use $n \in \{1, \dots, 4\}$.

For ACS data, we also consider the difference of predictive utility between
binary classifiers learned either on $\D$ or $\Dsyn$ and evaluated on a test
set $\Dtst \sim \lawdat$, following Equation~\ref{eq:utility-parity}.
We report the area under the receiver-operator curve (AUROC), and refer to
our appendices for additional metrics (area under the recall-precision curve,
accuracy and balanced accuracy).

\section{Experimental Results and Discussion}\label{sec:results}

In this section, we present key experimental results and discuss their
implications for causes of disparate impact of SDG. We first present
results on controlled distributions, then move on to the real-world ACS data.

\subsection{Results on Controlled Distributions}

Results on controlled distributions enable us to assess and formulate
hypotheses about structural causes for disparate impact. In Figures~\ref{fig:toy-tvd-base} and \ref{fig:toy-tvd-dp}
we report $\delta_{\textnormal{TVD}} := \textnormal{TVD}(\Dsyn_1, \D_1)
- \textnormal{TVD}(\Dsyn_0, \D_0)$ values, that measure the disparate fidelity
of a synthetic dataset. High absolute values of $\delta_{\textnormal{TVD}}$
indicate disparate impact.
Positive values indicate that group $S$=1 is disadvantaged, and conversely.
We use boxplots of the distribution of example-wise
$\delta_{\textnormal{TVD}}$ values for each method and setting,
so as to illustrate the variance of $\delta_{\textnormal{TVD}}$ across
examples.
Additional figures that extensively cover the studied methods
are provided in Appendix~\ref{appendix:tvd-toy}.

We note that results sometimes vary extensively across examples within
a setting, which is due to the randomness of the values of conditional
probabilities that define example distributions in addition to their shared
graphical structure.
Notionally, if a SDG method does not model the dependency of $B$ on $A$, the
resulting error for marginal probabilities of $\lawsyn$ relative to $\lawdat$
depends on how much $\lawdat(B|A=1)$ and $\lawdat(B|A=0)$ values actually
differ.
Conversely, these values may influence whether the graph-selection process of
SDG methods succeeds in learning the underlying structure of $\lawdat$.
Notwithstanding, we can safely compare settings due to examples using the
same random conditionals across them.

\subsubsection{Disparate Impact}

We first compare non-DP, population-wide methods, including
PrivBayes($\epsilon=100$), that has loose privacy guarantees.
Figure~\ref{fig:toy-tvd-base-popwide} displays the distribution of
$\delta_{\textnormal{TVD}}$ across example distributions for each setting
and SDG method. Each subplot corresponds to a method (with PB and GB denoting
PrivBayes and GreedyBayes respectively). The x-axis denotes our four settings
(Base, Fewer-samples, Higher-complexity, Double-disadvantage).
The y-axis is shared by subplots on the same row.
Small circles denote outliers.

\begin{figure}[!t]
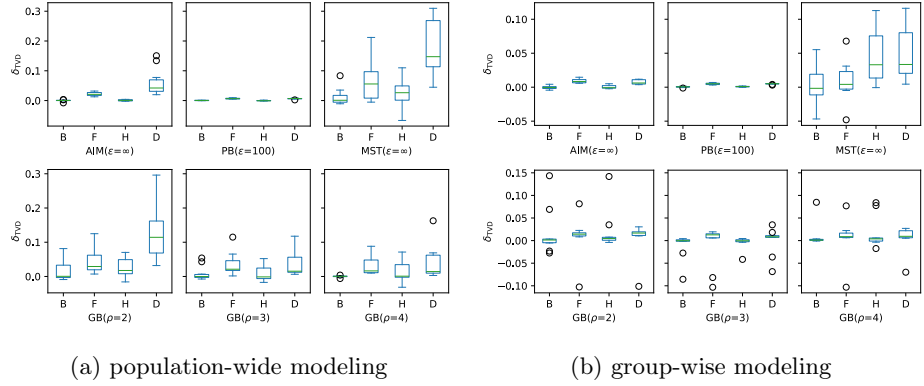

  \centering
  \begin{subfigure}[t]{0.5\textwidth}
    \includesvg[width=\textwidth]{plots/toy_tvd_dsp_base_popwide}
    \caption{population-wide modeling}
    \label{fig:toy-tvd-base-popwide}
  \end{subfigure}%
  ~
  \begin{subfigure}[t]{0.5\textwidth}
    \includesvg[width=\textwidth]{plots/toy_tvd_dsp_base_groupwise}
    \caption{group-wise modeling}
    \label{fig:toy-tvd-base-groupwise}
  \end{subfigure}%
  \caption{$\delta_{\textnormal{TVD}}$ of non-private SDG methods PrivBayes (PB) and GreedyBayes (GB) for the settings Base (B), Fewer-samples (F), Higher-complexity (H), Double-disadvantage (D). }
  \label{fig:toy-tvd-base}
\end{figure}

We first observe that \emph{heterogeneous sensitive group proportions cause
disparate estimation errors}. Indeed, $\delta_{\textnormal{TVD}} > 0$ for all
methods in the fewer-samples and double-disadvantage settings (boxplots F and D
for each method), whereas in the other two settings (boxplots B and H) there
are disparities in favor of either group depending on examples.
Even PrivBayes($\epsilon$=100), which achieves near-perfect utility and is
$0.01$-fair for all examples, has a slight disparate impact that systematically
disfavors the minority group.

We also observe that \emph{heterogeneous distribution complexity across groups
causes disparate approximation errors}.
Indeed, MST, which is subject to approximation errors in all settings,
rarely has disparate impact in the base setting, but
near-always does in the higher-complexity setting, in favor of the group
with the simpler distribution. This is also true for
GreedyBayes($\rho$=2), which in the higher-complexity setting can learn
$\lawdat_1$ but not $\lawdat_0$.

We then observe that \emph{heterogeneous group proportions and distribution
complexity are cumulative causes of disparate impact}.
Indeed, we see that methods that suffer approximation errors
in the higher-complexity setting have even more disparate impact in the
double-disadvantage one, which we interpret as estimation and approximation
errors being cumulative.
Furthermore, we observe that non-DP AIM, which does not have disparate impact
in the base nor higher-degree settings, has a more disparate
impact in the double-disadvantage setting than in the fewer-samples one. We
posit that this is due to an algorithmic bias where disparate estimation errors
related to group sizes result in selecting a graphical structure that is
suitable for the majority group but not complex enough for the minority one.

Finally, we posit that \emph{heterogeneous $\lawdat_k$ complexity across groups
can cause disparate estimation errors}. Indeed, in settings where group
proportions are equal (base and higher-degree), GreedyBayes often selects
a graph structure that does not fully match $\lawdat$. While this alternatively
advantages either group, we see that $|\delta_{\textnormal{TVD}}|$ is larger
when group $S$=1 is disadvantaged, as illustrated by the asymmetry of boxplots.
We posit that this is due to $\lawdat_1$ having higher complexity, so that graph
estimation errors have more impact on utility.

\begin{figure}[!t]
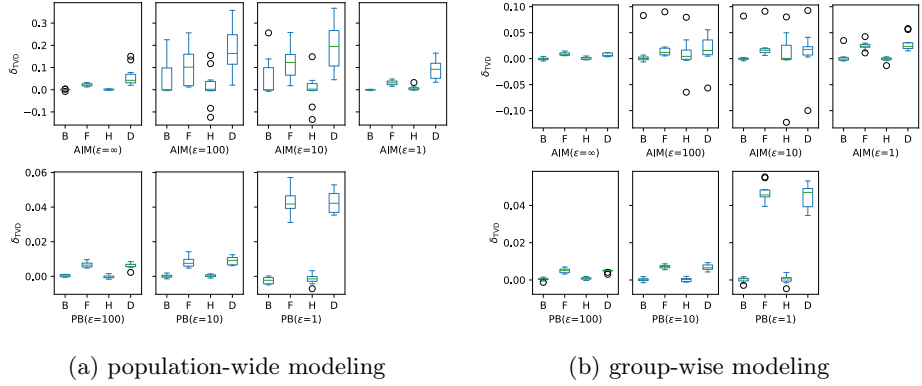

  \centering
  \begin{subfigure}[t]{0.5\textwidth}
    \includesvg[width=\textwidth]{plots/toy_tvd_dsp_dp_popwide}
    \caption{population-wide modeling}
    \label{fig:toy-tvd-dp-popwide}
  \end{subfigure}%
  ~
  \begin{subfigure}[t]{0.5\textwidth}
    \includesvg[width=\textwidth]{plots/toy_tvd_dsp_dp_groupwise}
    \caption{group-wise modeling}
    \label{fig:toy-tvd-dp-groupwise}
  \end{subfigure}%
  \caption{$\delta_{\textnormal{TVD}}$ of AIM and PrivBayes (PB)
  for various DP budgets  for the settings Base (B), Fewer-samples (F), Higher-complexity (H), Double-disadvantage (D).}
  \label{fig:toy-tvd-dp}
\end{figure}

\subsubsection{Impact of DP}

Figure~\ref{fig:toy-tvd-dp-popwide} represents distributions of
$\delta_{\textnormal{TVD}}$ for AIM and PrivBayes for various DP budgets,
with a similar formalism as Figure~\ref{fig:toy-tvd-base-popwide}.
It illustrates that \emph{DP has disparate impact when group proportions are
heterogeneous}, with both PrivBayes and AIM advantaging the majority group in
the fewer-results and double-disadvantage settings.
We posit that this is due to heterogeneous signal-to-noise ratio across groups,
resulting in disparate estimation errors.

Furthermore, we observe that \emph{DP can have disparate impact when
distribution complexity is heterogeneous}. Indeed, AIM($\epsilon$=1) has
disparate impact in the higher-complexity setting whereas its non-DP
counterpart does not. We posit that this is due to DP influencing the
graph-selection process, which is more likely to be harmful for groups
with a more complex distribution. We note that PrivBayes does not suffer
from this effect, which highlights that method specificities can cause
disparate impact in conjunction with distribution properties.

We remark that AIM is subject to another form of DP-induced algorithmic bias:
in our controlled settings, when using $\epsilon \in \{10, 100\}$, an annealing
mechanism that is specific to AIM is triggered, resulting in an early-exit that
yields an under-informed $\lawsyn$ that has lower utility and more disparate
impact favoring $S$=1 than in the non-DP and high-privacy regimes.

Finally, we report that DP has no visible impact on MST, which we interpret as
approximation errors dominating estimation ones for this simple model.

\subsubsection{Group-wise modeling}

Figure~\ref{fig:toy-tvd-base-groupwise} represents distributions of
$\delta_{\textnormal{TVD}}$ when using group-wise modeling with the
same base methods as presented in Figure~\ref{fig:toy-tvd-base-popwide}.
Complementarily, Figure~\ref{fig:toy-tvd-dp-groupwise} presents results
for group-wise modeling with AIM and PrivBayes for various DP budgets.

Generally-speaking, we observe from the comparison of Figure~\ref{%
fig:toy-tvd-base-popwide} and Figure~\ref{fig:toy-tvd-base-groupwise}
that \emph{group-wise modeling reduces the disparate impact of SDG}.
We notably remark that in the double-disadvantage setting, results for
group-wise modeling using non-DP AIM, GreedyBayes and PrivBayes($\epsilon$=100)
are similar to those in the fewer-samples setting, contrary to what was
observed with population-wide modeling.
We posit that this is due to group-wise modeling preventing algorithms from
selecting a distribution structure matching only the majority group. Therefore,
we believe that the low disparate impact of group-wise SDG methods are only due
to disparate estimation errors of marginal statistics caused by heterogeneous
group proportions, whereas the disparate impact of population-wide methods
is also due to disparate estimation errors in the graph-selection process.

However, we also observe that \emph{DP has a strong disparate impact on
group-wise modeling}. Indeed, in the high-privacy regime, group-wise modeling
achieves lesser utility, and has disparate impact disfavoring the minority
group when proportions are heterogeneous. This is explained by the
heterogeneous signal-to-noise ratio across groups, and by the fact that
statistics common to multiple $\lawsyn_k$ are estimated multiple times with
varying statistical power.

\subsection{Results on ACS Data}

Results on ACS data enable us to illustrate that disparate impact can happen
in real-world applications, and to identify practical cases of the phenomena
that we analyzed using artificial data in controlled settings.

\subsubsection{Disparate Impact on Distribution Fidelity}

\begin{table}[!t]
\centering
\caption{
  Average TVD between 2-way marginals of synthetic and real ACS data,
  computed for each and every SDG method and sensitive group: White Male (WM),
  White Female (WF), Black Male (BM), Black Female (BF).
}
\label{table:tvd-2way-acs-main}
\begin{tabularx}{0.95\textwidth}{|ll|CCCC|CCCC|}
\cline{1-10} &  & \multicolumn{4}{c|}{population-wide SDG} & \multicolumn{4}{c|}{groupwise SDG} \\
 &  & WM & WF & BM & BF & WM & WF & BM & BF \\
\cline{1-10}\multirow[c]{4}{*}{AIM} & $\epsilon$=1 & 2.86 & 3.03 & 8.60 & 8.80 & 2.57 & 2.82 & 8.12 & 8.51 \\
 & $\epsilon$=10 & 1.95 & 2.12 & 8.22 & 8.10 & 1.41 & 1.40 & 4.44 & 4.59 \\
 & $\epsilon$=100 & 2.70 & 3.08 & 11.04 & 11.11 & 1.41 & 1.54 & 2.92 & 2.63 \\
 & $\epsilon$=$\infty$ & 3.17 & 3.48 & 12.03 & 11.21 & 1.55 & 1.72 & 2.98 & 3.16 \\
\cline{1-10}
\multirow[c]{3}{*}{GreedyBayes} & $\rho$=2 & 5.69 & 5.14 & 15.58 & 13.81 & 3.38 & 3.41 & 4.52 & 4.22 \\
 & $\rho$=3 & 4.18 & 4.12 & 11.15 & 12.65 & 2.26 & 2.08 & 3.91 & 3.67 \\
 & $\rho$=4 & 3.25 & 3.42 & 10.06 & 10.03 & 1.48 & 1.29 & 4.01 & 2.83 \\
\cline{1-10}
\multirow[c]{4}{*}{MST} & $\epsilon$=1 & 7.85 & 7.19 & 16.32 & 16.59 & 5.62 & 7.52 & 16.42 & 15.43 \\
 & $\epsilon$=10 & 7.57 & 6.96 & 16.05 & 16.34 & 4.84 & 5.10 & 7.80 & 7.13 \\
 & $\epsilon$=100 & 7.59 & 6.99 & 16.00 & 16.41 & 4.81 & 5.02 & 6.94 & 6.36 \\
 & $\epsilon$=$\infty$ & 7.55 & 6.94 & 16.04 & 16.39 & 4.81 & 5.00 & 6.91 & 6.32 \\
\cline{1-10}
\multirow[c]{3}{*}{PrivBayes} & $\epsilon$=1 & 5.72 & 5.81 & 16.12 & 14.99 & 5.64 & 6.02 & 10.78 & 10.93 \\
 & $\epsilon$=10 & 5.30 & 5.15 & 14.05 & 13.28 & 3.64 & 3.95 & 6.83 & 7.22 \\
 & $\epsilon$=100 & 3.93 & 4.42 & 14.31 & 13.74 & 3.16 & 3.55 & 6.90 & 6.07 \\
\cline{1-10}
\end{tabularx}
\end{table}

We report in Table~\ref{table:tvd-2way-acs-main} the
group-wise average TVD between the 2-way empirical marginals of synthetic
and real data for all considered SDG methods, scaled by 100. Higher TVD values
indicate lower fidelity of $\Dsyn_k$ to $\D_k$. Disparate impact is
characterized by discrepancies between the four group-wise values of any
given method.
TVD metrics for other $n$-way marginals and over the full distribution
are reported in Appendix~\ref{appendix:tvd-acs}.
We do not report standard deviations, that are of low magnitude.

These results show that all SDG methods have a disparate impact on the fidelity
of $\Dsyn_k$ to $\D_k$ that disadvantages black people. Methods with a broader
hypothesis class tend to have both better utility and less disparate impact.
For most methods, the group-wise modeling strategy greatly improves utility
for all groups and reduces disparity, but does not prevent it. Furthermore,
it does worse than population-wide modeling in the high-privacy regime
($\epsilon=1$).

These results are coherent with the expectation that broadening the hypothesis
class of SDG methods reduces the approximation error for all groups, and that
estimation errors put groups with fewer samples at a disadvantage. We also
observe the disparate impact of DP on groups of different sizes in
high-privacy regimes, which is the strongest with group-wise modeling.

\subsubsection{Disparate Impact on Downstream Utility}

\begin{table}[t!]
\centering
\caption{
  Decrease of AUROC when training from synthetic rather than real ACS data,
  evaluated on a real-world testing set and averaged across replicas,
  computed for each and every SDG method and sensitive group: all groups (all),
  White Male (WM), White Female (WF), Black Male (BM), Black Female (BF).
}
\label{table:auroc-diff-acs-main}
\begin{tabularx}{0.95\textwidth}{|ll|CCCCC|CCCCC|}
\cline{1-12} &  & \multicolumn{5}{c|}{population-wide SDG} & \multicolumn{5}{c|}{groupwise SDG} \\
 &  & all & WM & WF & BM & BF & all & WM & WF & BM & BF \\
\cline{1-12}\multirow[c]{4}{*}{AIM} & $\epsilon$=1 & 1.34 & 1.20 & 1.56 & 2.04 & 1.91 & 1.47 & 1.36 & 1.56 & 2.55 & 2.71 \\
 & $\epsilon$=10 & 0.90 & 0.77 & 1.05 & 1.37 & 1.27 & 0.68 & 0.62 & 0.77 & 1.21 & 0.75 \\
 & $\epsilon$=100 & 1.00 & 0.90 & 1.09 & 1.38 & 1.40 & 0.56 & 0.55 & 0.63 & 0.91 & 0.49 \\
 & $\epsilon$=$\infty$ & 0.91 & 0.74 & 1.08 & 1.34 & 1.48 & 0.88 & 0.68 & 1.18 & 1.41 & 0.82 \\
\cline{1-12}
\multirow[c]{3}{*}{GreedyBayes} & $\rho$=2 & 3.69 & 3.85 & 3.81 & 4.77 & 4.66 & 3.57 & 3.70 & 3.56 & 6.33 & 5.50 \\
 & $\rho$=3 & 4.86 & 5.42 & 3.89 & 6.14 & 4.55 & 3.04 & 2.74 & 3.35 & 8.26 & 4.01 \\
 & $\rho$=4 & 2.91 & 3.06 & 2.93 & 4.02 & 3.52 & 1.18 & 1.37 & 1.00 & 1.94 & 2.02 \\
\cline{1-12}
\multirow[c]{4}{*}{MST} & $\epsilon$=1 & 9.39 & 7.93 & 10.35 & 8.03 & 9.95 & 9.05 & 7.33 & 12.79 & 16.67 & 11.92 \\
 & $\epsilon$=10 & 9.39 & 7.89 & 10.37 & 8.05 & 9.97 & 8.97 & 7.30 & 12.65 & 15.91 & 11.58 \\
 & $\epsilon$=100 & 9.39 & 7.91 & 10.32 & 7.98 & 9.93 & 9.01 & 7.30 & 12.78 & 16.28 & 11.62 \\
 & $\epsilon$=$\infty$ & 9.40 & 7.94 & 10.36 & 8.02 & 9.93 & 8.97 & 7.29 & 12.69 & 16.17 & 11.43 \\
\cline{1-12}
\multirow[c]{3}{*}{PrivBayes} & $\epsilon$=1 & 2.30 & 2.23 & 2.44 & 2.50 & 2.31 & 2.33 & 2.21 & 2.46 & 3.84 & 3.28 \\
 & $\epsilon$=10 & 1.44 & 1.36 & 1.60 & 2.41 & 1.80 & 1.25 & 1.23 & 1.26 & 2.50 & 1.76 \\
 & $\epsilon$=100 & 1.67 & 1.53 & 1.89 & 2.18 & 2.28 & 1.07 & 0.85 & 1.12 & 3.41 & 2.82 \\
\cline{1-12}
\end{tabularx}
\end{table}

We report in Table~\ref{table:auroc-diff-acs-main} the population-wide and
group-wise difference of AUROC between a classifier learned from real-world data
and one learned from a synthetic dataset, evaluated on a real-world testing
set, averaged across replicas and scaled by 100.
Positive values indicate a decrease of AUROC (hence, of utility) when using
synthetic rather than real-world training data. Disparate impact is
characterized by discrepancies between the four group-wise values of any
given method. The impact over the entire population is also reported for
reference.
Similar results for additional evaluation metrics are reported in
Appendix~\ref{appendix:clf-acs}, where utility values can also be found.
We do not report standard deviations, that are of low magnitude.

These results show that SDG methods have a variable disparate impact on the
utility of downstream classifiers learned from synthetic rather than real-world data.
The loss of utility is usually higher for black than white people, but the
ordering of groups from least- to most-impacted varies across methods and
sometimes across parameter-variants of a method.
Broadening the hypothesis class of population-wide methods has a clear effect
on preserving the downstream utility for all groups, and usually reduces the
disparate impact across groups.

Group-wise modeling exhibits contrasted results. For AIM,
which is the method achieving the best utility and least disparate impact,
group-wise modeling further improves both utility and fairness, even making
black women slightly advantaged compared to white women and black men.
Conversely, for most other methods, the group-wise modeling strategy
increases disparate impact, by reducing the loss of utility for some groups
(white men, sometimes white women) and increasing it for others (sometimes
disparately across disadvantaged groups).

Results for population-wide modeling methods are coherent with our expectations
and with results on distances across synthetic and real distributions.
For group-wise modeling, we hypothesize that its disparate impact across
groups arises from the reduction of estimation and approximation errors
objectified by TVD metrics being only partial, so that group-wise models
may miss some information that is better-modeled by population-wide ones
and that would be more useful to learn the classifiers.
At any rate, our results illustrate that utility metrics for SDG are not
necessarily aligned. This motivates our formalization of disparate impact
for SDG as a broad framework, to be instantiated for utility metrics that
depend on the application context.

\section{Conclusion}

In this paper, we have introduced a new fairness definition for synthetic data
generation (SDG), by revisiting the well-established fairness notion of
disparate impact. We define SDG to be fair when the output synthetic data is
of equal utility across sensitive groups.
We believe that this problem is both under-explored and crucial for real-world
applications of SDG as a privacy-preserving technology.
We hope that formalizing it can influence both researchers to work on
understanding and addressing unfairness, and real-world practitioners
to consider fairness when assessing the utility of synthetic data.

We illustrated that SDG methods that rely on a probabilistic graphical model
(PGM) can have disparate impact and explored reasons why they do, using both
artificial and real-world data.
Our experiments on distributions with controlled properties enabled us to
validate that disparate impact can have multiple and cumulative sources.
We showed that unbalanced samples sizes across groups can cause disparate
estimation errors, and can be amplified by differential privacy (DP)
mechanisms.
Furthermore, we showed that heterogeneous distribution complexity, in the
sense that some group-conditional distributions are easier to learn than others
by the SDG method, can cause disparate impact, even when the true distribution
is part of the hypothesis class of that method. We highlighted that the methods
we studied are subject to graph-selection estimation errors, that can be
influenced by group proportions and DP, but also arise from method-specific
algorithms.

Additionally, we introduced a meta-algorithm for SDG that relies on learning
group-wise models.
Our experiments demonstrated that this strategy tends to decrease the disparate
impact of SDG methods by increasing utility for otherwise-disadvantaged
groups, which appears to be due to avoiding some graph estimation errors.
However, this strategy is prone to disparate estimation errors, notably when
using DP, and appears to be unsuited for high-privacy regimes.

Finally, our experiments on real-world ACS data highlighted that while disparate
impact is prone to occur, it may vary across considered utility metrics, as does
the overall utility.
This highlights the importance for practitioners to evaluate synthetic data
generation in coherence with its planned uses, which are highly-contextual.
Similarly, we suggest that the tolerance to unfairness should be set in
accordance with the application context and the potential for harm it creates.
While this may appear to contradict the selling point of synthetic data as a
task-agnostic privacy-preserving solution for data release, we argue that it
is coherent with the fact that properties of the modeled data distribution are
variously useful depending on applications. For instance, if a feature that is
uncorrelated with the target one is badly modeled, it is likely to have low or
no impact on the utility of downstream ML models, whereas it may make the data
appear unrealistic or even prejudiced to the eyes of human practitioners.

\subsubsection{Future work}

Our work paves the way towards future research that may address how to prevent
or limit disparate impact in SDG. Our group-wise strategy acts as a natural
baseline, that may notably be challenged by leveraging the statistical power of
the complete dataset rather than processing group-wise subsets in a
merely-parallel way.
Another possible line of work would be to understand more precisely why methods
sometimes fail to select a PGM suitable for all groups despite the true
distribution being part of their hypothesis class, and to design improved
graph-selection heuristics.
Finally, our fairness definition may be applied to SDG methods that are not
based on PGMs, and to data beyond the tabular categorical setting that we used.

Tackling SDG methods that rely on deep generative models is bound to call for
a renewed analysis of the mechanisms underlying disparate impact. While
we expect the decomposition of disparate impact sources between approximation
and various kinds of estimation errors to remain relevant and to ground similar
conclusions as those we drew for PGM-based methods, we also expect this
decomposition to be harder to objectify. Indeed, most deep generative models
lack an explicit interpretable hypothesis class, and learn a transform between
an arbitrary input distribution and an output synthetic distribution rather
than learning the latter explicitly, making it harder to reason about which
distributions are learnable by the methods and to ensure which dependencies
across attributes have been modeled.

\begin{credits}
  \subsubsection{\ackname}
  This work was supported by the ANR 22-CMAS-0009 CAPS'UL
  (CAmpus Participatif en Santé numérique du Site Universitaire de Lille)
  project of the France 2030 AMI-CMA, and the ANR 22-PECY-0002 IPOP
  (Interdisciplinary Project on Privacy) project of the Cybersecurity PEPR.
  Michaël Perrot is supported by the French National Research Agency (ANR)
  through the grant ANR-23-CE23-0011-01 (Project FaCTor).

  \subsubsection{\discintname}
  The authors have no competing interests to declare that are relevant to the
  content of this article.
\end{credits}

\bibliographystyle{splncs04}
\bibliography{bibliography}

\clearpage
\appendix

\section{Experimental Setup Details}\label{appendix:setup}

In this section, we provide some specific methodological details as to our
experiments. Note that we also provide our entire source code as supplementary
material, including backend SDG methods, experiments-defining scripts, random
seeds that we used and scripts that produced the tables and figures presented
in the paper and its appendices.

\subsection{SDG methods parameters}\label{appendix:setup-sdg}

We set AIM to use a workload based on all possible 3-attributes cliques
and to take measurements over cliques of at most 3 attributes. We use the
default maximum tree size of 80 MB from AIM authors, but limit the number
of planned measurements to $4d$ on real-world data and $2d$ on low-dimensional
artificial data, where $d$ is the number of attributes (respectively 12 and 6
for population-wide models over ACS or artificial data, and 10 and 5 for
group-wise models that do not treat sensitive attributes). We also limit
the computational effort put in selecting a distribution based on noisy
measurements at each step, by running at most 200 gradient descent steps
for ACS data, and 100 ones for artificial data.

We use default parameters for PrivBayes and MST taken from the papers
they were respectfully introduced in.

\subsection{Custom non-DP SDG methods}\label{appendix:setup-nondp}

We use custom implementations for all SDG methods, that are based on both
the papers introducing them and existing reference implementations. This
enables us to implement non-DP variants of MST and AIM.

For MST, we simply use an argmax over scores to select the graph structure
rather than a DP exponential mechanism, and do not add noise to marginals
measured over the training data.

For AIM, we also use an argmax over scores to select the clique over which
to take a measurement at each step, and do not add noise to these measurements.
As a consequence, the annealing mechanism from the base algorithm is disabled,
and the planned number of measurements is always taken unless the junction
tree becomes oversized. We keep the rest of the algorithm structure unchanged,
including the use of gradient descent to select a distribution that matches
measured marginals. We use exact total measurements, and pass an arbitrary
mock value of $\sigma = 1$ for all measurements that result in the minimized
loss not being re-scaled.

\subsection{Downstream Classifiers on ACS Data}\label{appendix:setup-acs}

As mentioned in the paper, each replicas of the ACS dataset is defined by a
unique, uniformly-random split of the base dataset (roughly 490 000 records)
into three sets, respectively labeled as training (80~\% of records),
validation (10~\% of records) and test (10~\% of records) sets.

For each replica, SDG methods are run to produce 500 000 synthetic records
based on the training set. XGBoost binary classifiers are then learned either
from real-world or synthetic data.
When using real-world data, a grid-search is used to tune a regularization
parameter based on generalization AUROC over the validation set. The selected
parameter is used to train a final classifier from the union of training and
validation sets.
When using synthetic data, a similar grid-search is performed using 20~\% of
synthetic records for validation. The final classifier is trained on the full
synthetic dataset.

The XGBoost classifiers are implemented using the \texttt{xgboost} Python
library, with default hyper-parameters save for the learning rate (set to 0.1),
the enabling of processing categorical variables as such, the scaling of
the loss by $P(Y = 0) / P(Y = 1)$ to account for label imbalance (using
estimates from the training data, either synthetic or real) and the L2
regularization constant, which is selected using a grid-search among candidate
values $[0, 1, 4, 8, 16, 32]$.

\newpage

\section{Expanded Results on Controlled Distributions}\label{appendix:tvd-toy}

In this section, we report exhaustive results on disparate impact fairness
assessed based on distances between empirical distributions of synthetic and
artificial data sampled from controlled distributions.

We report $\delta_{\textnormal{TVD}} := \textnormal{TVD}(\Dsyn_1, \D_1)
- \textnormal{TVD}(\Dsyn_0, \D_0)$ values, that measure the disparate utility
of a synthetic dataset. Disparate impact is characterized by high absolute
values of $\delta_{\textnormal{TVD}}$. Positive values indicate that
group $S$=1 is disadvantaged, and conversely.

We provide boxplots that represent the distribution of
$\delta_{\textnormal{TVD}} = $ across example distributions for variants of
a given SDG method across controlled settings.
Each subplot corresponds to a controlled setting.
The x-axis denotes variants of a SDG method.
The y-axis is shared by subplots on the same row.
Small circles denote outliers.

We present results for
AIM (Figure~\ref{fig:toy-tvd-aim-popwide}),
group-wise AIM (Figure~\ref{fig:toy-tvd-aim-groupwise}),
PrivBayes (Figure~\ref{fig:toy-tvd-privbayes-popwide}),
group-wise PrivBayes (Figure~\ref{fig:toy-tvd-privbayes-groupwise}),
GreedyBayes (Figure~\ref{fig:toy-tvd-greedybayes-popwide}),
group-wise GreedyBayes (Figure~\ref{fig:toy-tvd-greedybayes-groupwise}),
MST (Figure~\ref{fig:toy-tvd-mst-popwide}),
and group-wise MST (Figure~\ref{fig:toy-tvd-mst-groupwise}).

\begin{figure}
  \centering
  \includesvg[width=0.8\textwidth]{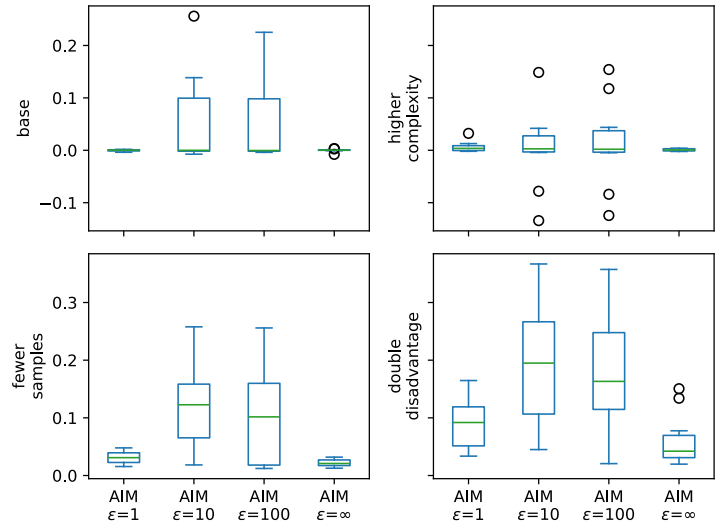}
  \caption{$\delta_{\textnormal{TVD}}$ for population-wide AIM methods}
  \label{fig:toy-tvd-aim-popwide}
\end{figure}

\begin{figure}
  \centering
  \includesvg[width=0.8\textwidth]{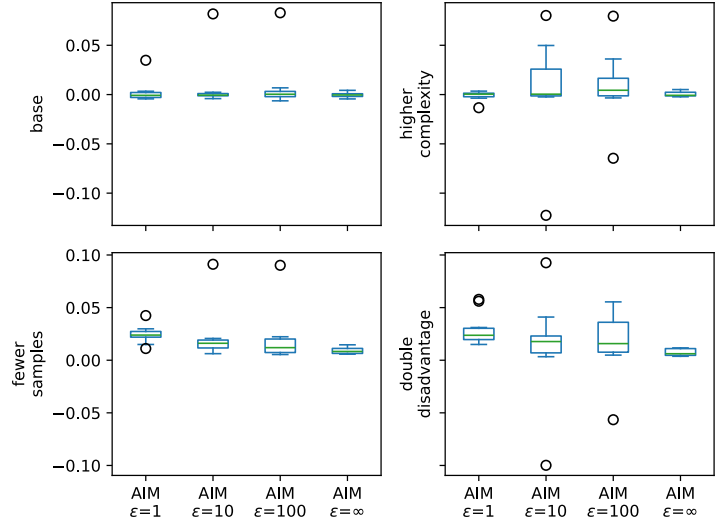}
  \caption{$\delta_{\textnormal{TVD}}$ for group-wise AIM methods}
  \label{fig:toy-tvd-aim-groupwise}
\end{figure}

\begin{figure}
  \centering
  \includesvg[width=0.8\textwidth]{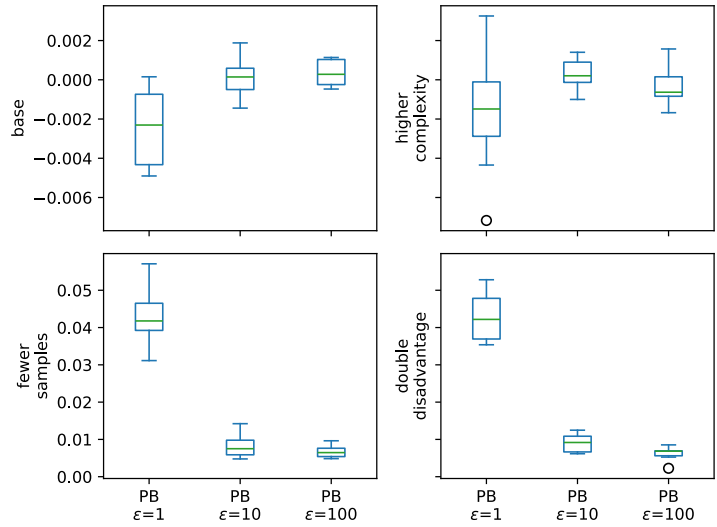}
  \caption{$\delta_{\textnormal{TVD}}$ for population-wide PrivBayes methods}
  \label{fig:toy-tvd-privbayes-popwide}
\end{figure}

\begin{figure}
  \centering
  \includesvg[width=0.8\textwidth]{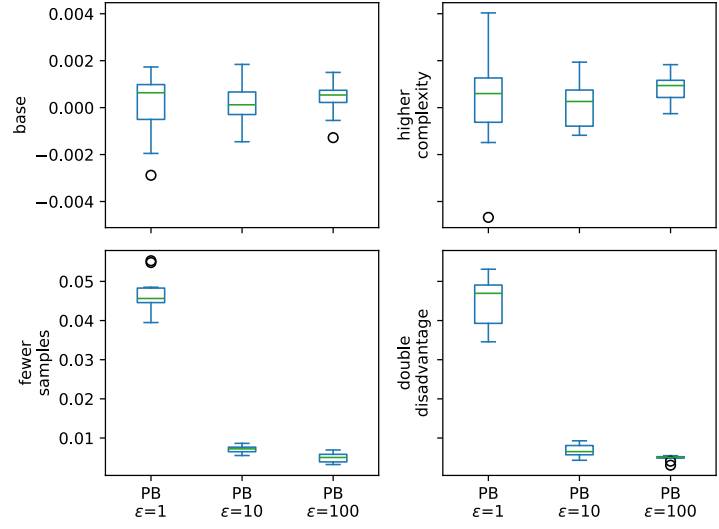}
  \caption{$\delta_{\textnormal{TVD}}$ for group-wise PrivBayes methods}
  \label{fig:toy-tvd-privbayes-groupwise}
\end{figure}

\begin{figure}
  \centering
  \includesvg[width=0.8\textwidth]{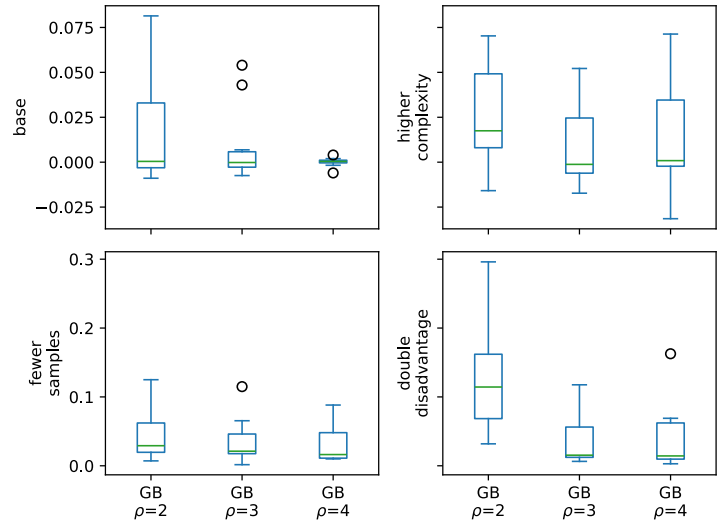}
  \caption{$\delta_{\textnormal{TVD}}$ for population-wide GreedyBayes methods}
  \label{fig:toy-tvd-greedybayes-popwide}
\end{figure}

\begin{figure}
  \centering
  \includesvg[width=0.8\textwidth]{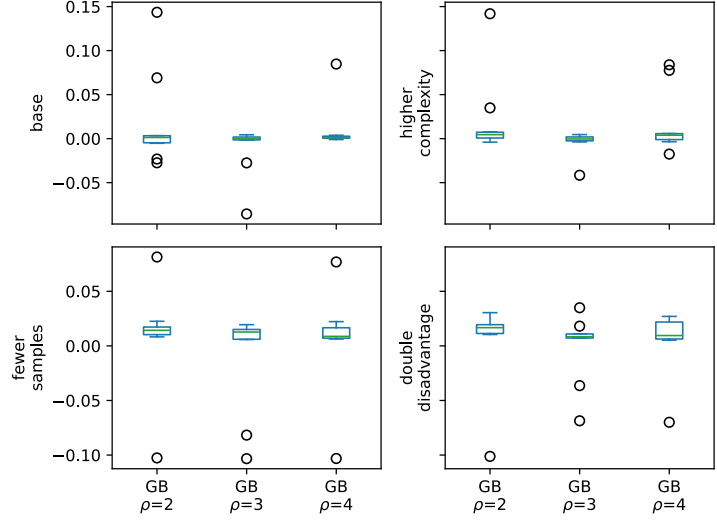}
  \caption{$\delta_{\textnormal{TVD}}$ for group-wise GreedyBayes methods}
  \label{fig:toy-tvd-greedybayes-groupwise}
\end{figure}

\begin{figure}
  \centering
  \includesvg[width=0.8\textwidth]{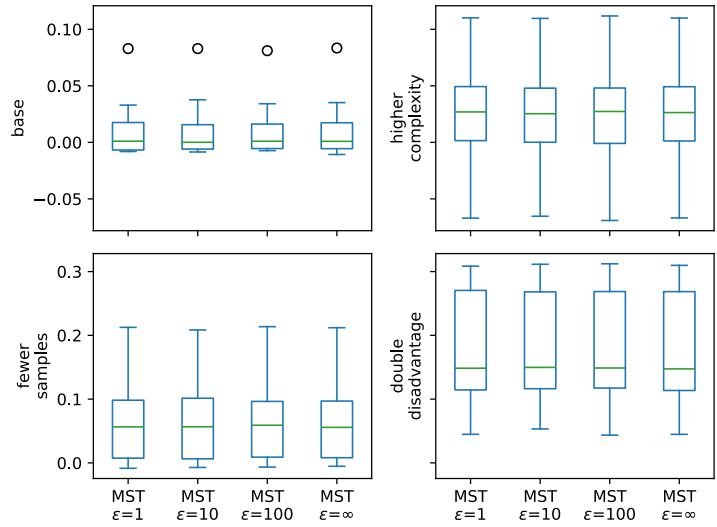}
  \caption{$\delta_{\textnormal{TVD}}$ for population-wide MST methods}
  \label{fig:toy-tvd-mst-popwide}
\end{figure}

\begin{figure}
  \centering
  \includesvg[width=0.8\textwidth]{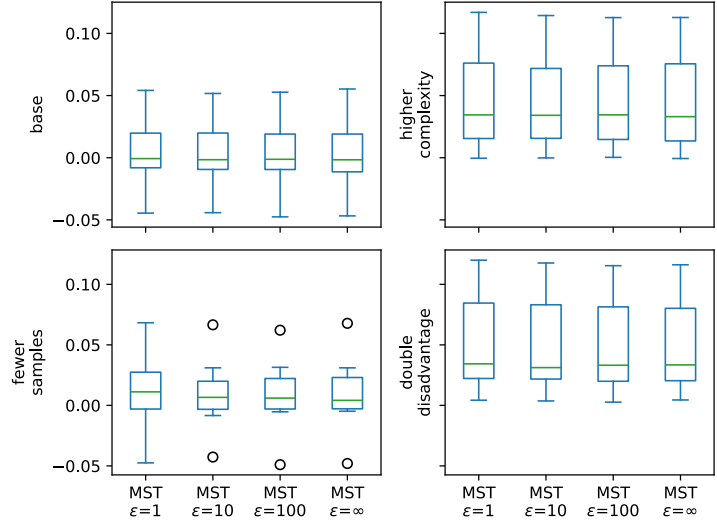}
  \caption{$\delta_{\textnormal{TVD}}$ for group-wise MST methods}
  \label{fig:toy-tvd-mst-groupwise}
\end{figure}

\newpage

\section{Expanded Results on ACS Data}

\subsection{Synthetic Distribution Fidelity}\label{appendix:tvd-acs}

In this section, we report results on disparate impact fairness assessed based
on distances between empirical distributions of synthetic and real-world ACS
data.

We report the group-wise average TVD between the $n$-way empirical marginals
of synthetic and real data for all considered SDG methods, scaled by 100.
Higher TVD values indicate lower fidelity of $\Dsyn_k$ to $\D_k$. Disparate
impact is characterized by discrepancies between the four group-wise values
of any given method.
Values for $n=2$ are presented in Table~\ref{table:tvd-2way-acs}
(identical to Table~\ref{table:tvd-2way-acs-main} in the main paper),
values for $n=3$ are in Table~\ref{table:tvd-3way-acs},
and values for $n=4$ are in Table~\ref{table:tvd-4way-acs}.

We also report in Table~\ref{table:tvd-global-acs} the TVD between
synthetic and marginal distributions, averaged across replicas and scaled by
100. Due to the total domain size of the joint distribution, results are not
as easily compared as with $n$-way marginals.

\input{tables/acs_tvd_1way}
\input{tables/acs_tvd_2way}
\input{tables/acs_tvd_3way}
\input{tables/acs_tvd_4way}
\input{tables/acs_tvd_global}

\newpage

\subsection{Downstream Classifier Utility}\label{appendix:clf-acs}

In this section, we report results on disparate impact fairness assessed based
on the loss of utility from training a binary XGBoost classifier from synthetic
data rather than from real-world ACS data.

We present the difference between a utility metric for the model learned on
real data and that for the model learned on synthetic data, evaluated on a
real-world testing dataset, averaged across replicas and scaled by 100.
Negative values indicate a loss of utility when using synthetic rather than
real-world training data. Disparate impact is characterized by discrepancies
between the four group-wise values of any given method.

We report differences for various metrics:
AUROC (Table~\ref{table:auroc-diff-acs}, identical to
Table~\ref{table:auroc-diff-acs-main} from the main paper),
AURPC (Table~\ref{table:aurpc-diff-acs}),
Accuracy (Table~\ref{table:accuracy-diff-acs}),
and Balanced Accuracy (Table~\ref{table:bal_acc-diff-acs}).
The latter two are computed using a binary prediction threshold of 0.5.

For reference, we also report raw utility metrics of the models learned from
synthetic data. We report various metrics:
AUROC (Table~\ref{table:auroc-acs}),
AURPC (Table~\ref{table:aurpc-acs}),
Accuracy (Table~\ref{table:accuracy-acs}),
and Balanced Accuracy (Table~\ref{table:bal_acc-acs}).

\input{tables/acs_auroc_diff}
\input{tables/acs_aurpc_diff}
\input{tables/acs_bal_acc_diff}
\input{tables/acs_accuracy_diff}

\input{tables/acs_auroc}
\input{tables/acs_aurpc}
\input{tables/acs_bal_acc}
\input{tables/acs_accuracy}

\end{document}

%% file: tables/acs_tvd_1way.tex
\begin{table}
\centering
\caption{Average TVD between 1-way marginals of synthetic and real ACS data.}
\label{table:tvd-1way-acs}
\begin{tabularx}{0.95\textwidth}{|ll|CCCC|CCCC|}
\cline{1-10} &  & \multicolumn{4}{c|}{population-wide} & \multicolumn{4}{c|}{groupwise} \\
 &  & WM & WF & BM & BF & WM & WF & BM & BF \\
\cline{1-10}\multirow[c]{4}{*}{AIM} & $\epsilon$=1 & 0.63 & 0.83 & 3.98 & 4.16 & 0.38 & 0.47 & 2.06 & 2.21 \\
 & $\epsilon$=10 & 0.43 & 0.57 & 3.86 & 3.82 & 0.24 & 0.23 & 0.54 & 0.44 \\
 & $\epsilon$=100 & 1.01 & 1.22 & 5.98 & 5.74 & 0.26 & 0.18 & 0.63 & 0.55 \\
 & $\epsilon$=$\infty$ & 1.36 & 1.53 & 6.84 & 5.99 & 0.30 & 0.30 & 0.49 & 0.44 \\
\cline{1-10}
\multirow[c]{3}{*}{GreedyBayes} & $\rho$=2 & 2.23 & 1.97 & 9.19 & 7.58 & 0.38 & 0.40 & 0.83 & 0.73 \\
 & $\rho$=3 & 1.84 & 1.95 & 6.07 & 7.11 & 0.25 & 0.27 & 0.78 & 0.80 \\
 & $\rho$=4 & 1.39 & 1.51 & 5.73 & 5.45 & 0.24 & 0.25 & 1.46 & 0.77 \\
\cline{1-10}
\multirow[c]{4}{*}{MST} & $\epsilon$=1 & 3.07 & 2.76 & 9.30 & 9.55 & 0.90 & 1.92 & 6.90 & 6.64 \\
 & $\epsilon$=10 & 2.95 & 2.70 & 9.21 & 9.35 & 0.28 & 0.31 & 1.44 & 1.20 \\
 & $\epsilon$=100 & 2.96 & 2.72 & 9.21 & 9.43 & 0.30 & 0.26 & 0.81 & 0.77 \\
 & $\epsilon$=$\infty$ & 2.95 & 2.66 & 9.18 & 9.42 & 0.24 & 0.25 & 0.85 & 0.81 \\
\cline{1-10}
\multirow[c]{3}{*}{PrivBayes} & $\epsilon$=1 & 1.99 & 1.98 & 9.20 & 8.86 & 1.85 & 1.91 & 3.27 & 3.55 \\
 & $\epsilon$=10 & 2.30 & 2.13 & 8.07 & 7.33 & 1.10 & 1.22 & 1.80 & 2.30 \\
 & $\epsilon$=100 & 1.70 & 1.98 & 8.44 & 8.05 & 1.10 & 1.25 & 2.45 & 2.01 \\
\cline{1-10}
\end{tabularx}
\end{table}

%% file: tables/acs_tvd_2way.tex
\begin{table}
\centering
\caption{Average TVD between 2-way marginals of synthetic and real ACS data.}
\label{table:tvd-2way-acs}
\begin{tabularx}{0.95\textwidth}{|ll|CCCC|CCCC|}
\cline{1-10} &  & \multicolumn{4}{c|}{population-wide} & \multicolumn{4}{c|}{groupwise} \\
 &  & WM & WF & BM & BF & WM & WF & BM & BF \\
\cline{1-10}\multirow[c]{4}{*}{AIM} & $\epsilon$=1 & 2.86 & 3.03 & 8.60 & 8.80 & 2.57 & 2.82 & 8.12 & 8.51 \\
 & $\epsilon$=10 & 1.95 & 2.12 & 8.22 & 8.10 & 1.41 & 1.40 & 4.44 & 4.59 \\
 & $\epsilon$=100 & 2.70 & 3.08 & 11.04 & 11.11 & 1.41 & 1.54 & 2.92 & 2.63 \\
 & $\epsilon$=$\infty$ & 3.17 & 3.48 & 12.03 & 11.21 & 1.55 & 1.72 & 2.98 & 3.16 \\
\cline{1-10}
\multirow[c]{3}{*}{GreedyBayes} & $\rho$=2 & 5.69 & 5.14 & 15.58 & 13.81 & 3.38 & 3.41 & 4.52 & 4.22 \\
 & $\rho$=3 & 4.18 & 4.12 & 11.15 & 12.65 & 2.26 & 2.08 & 3.91 & 3.67 \\
 & $\rho$=4 & 3.25 & 3.42 & 10.06 & 10.03 & 1.48 & 1.29 & 4.01 & 2.83 \\
\cline{1-10}
\multirow[c]{4}{*}{MST} & $\epsilon$=1 & 7.85 & 7.19 & 16.32 & 16.59 & 5.62 & 7.52 & 16.42 & 15.43 \\
 & $\epsilon$=10 & 7.57 & 6.96 & 16.05 & 16.34 & 4.84 & 5.10 & 7.80 & 7.13 \\
 & $\epsilon$=100 & 7.59 & 6.99 & 16.00 & 16.41 & 4.81 & 5.02 & 6.94 & 6.36 \\
 & $\epsilon$=$\infty$ & 7.55 & 6.94 & 16.04 & 16.39 & 4.81 & 5.00 & 6.91 & 6.32 \\
\cline{1-10}
\multirow[c]{3}{*}{PrivBayes} & $\epsilon$=1 & 5.72 & 5.81 & 16.12 & 14.99 & 5.64 & 6.02 & 10.78 & 10.93 \\
 & $\epsilon$=10 & 5.30 & 5.15 & 14.05 & 13.28 & 3.64 & 3.95 & 6.83 & 7.22 \\
 & $\epsilon$=100 & 3.93 & 4.42 & 14.31 & 13.74 & 3.16 & 3.55 & 6.90 & 6.07 \\
\cline{1-10}
\end{tabularx}
\end{table}

%% file: tables/acs_tvd_3way.tex
\begin{table}
\centering
\caption{Average TVD between 3-way marginals of synthetic and real ACS data.}
\label{table:tvd-3way-acs}
\begin{tabularx}{0.95\textwidth}{|ll|CCCC|CCCC|}
\cline{1-10} &  & \multicolumn{4}{c|}{population-wide} & \multicolumn{4}{c|}{groupwise} \\
 &  & WM & WF & BM & BF & WM & WF & BM & BF \\
\cline{1-10}\multirow[c]{4}{*}{AIM} & $\epsilon$=1 & 6.96 & 7.20 & 16.72 & 16.53 & 6.88 & 7.34 & 17.86 & 17.99 \\
 & $\epsilon$=10 & 5.26 & 5.51 & 15.83 & 15.27 & 4.51 & 4.59 & 12.82 & 12.62 \\
 & $\epsilon$=100 & 5.99 & 6.45 & 18.71 & 18.63 & 4.31 & 4.78 & 9.60 & 8.61 \\
 & $\epsilon$=$\infty$ & 6.35 & 6.82 & 19.91 & 18.68 & 4.50 & 4.83 & 8.78 & 9.21 \\
\cline{1-10}
\multirow[c]{3}{*}{GreedyBayes} & $\rho$=2 & 10.36 & 9.76 & 24.05 & 21.86 & 7.87 & 8.03 & 11.62 & 10.76 \\
 & $\rho$=3 & 7.88 & 7.50 & 18.57 & 19.86 & 5.76 & 5.16 & 9.86 & 9.20 \\
 & $\rho$=4 & 6.20 & 6.30 & 16.45 & 16.26 & 3.82 & 3.41 & 8.99 & 7.05 \\
\cline{1-10}
\multirow[c]{4}{*}{MST} & $\epsilon$=1 & 13.84 & 13.02 & 25.52 & 25.42 & 11.69 & 13.88 & 27.97 & 25.75 \\
 & $\epsilon$=10 & 13.45 & 12.69 & 25.18 & 25.10 & 10.71 & 11.10 & 17.35 & 15.80 \\
 & $\epsilon$=100 & 13.46 & 12.72 & 25.12 & 25.17 & 10.65 & 10.98 & 16.38 & 14.92 \\
 & $\epsilon$=$\infty$ & 13.42 & 12.67 & 25.13 & 25.13 & 10.66 & 10.98 & 16.34 & 14.88 \\
\cline{1-10}
\multirow[c]{3}{*}{PrivBayes} & $\epsilon$=1 & 11.01 & 11.15 & 25.06 & 23.63 & 11.09 & 11.71 & 21.51 & 21.04 \\
 & $\epsilon$=10 & 9.67 & 9.62 & 22.71 & 21.44 & 7.96 & 8.41 & 15.88 & 15.75 \\
 & $\epsilon$=100 & 7.61 & 8.35 & 22.66 & 21.46 & 6.78 & 7.49 & 15.22 & 13.52 \\
\cline{1-10}
\end{tabularx}
\end{table}

%% file: tables/acs_tvd_4way.tex
\begin{table}
\centering
\caption{Average TVD between 4-way marginals of synthetic and real ACS data.}
\label{table:tvd-4way-acs}
\begin{tabularx}{0.95\textwidth}{|ll|CCCC|CCCC|}
\cline{1-10} &  & \multicolumn{4}{c|}{population-wide} & \multicolumn{4}{c|}{groupwise} \\
 &  & WM & WF & BM & BF & WM & WF & BM & BF \\
\cline{1-10}\multirow[c]{4}{*}{AIM} & $\epsilon$=1 & 13.72 & 14.15 & 29.62 & 28.69 & 13.88 & 14.65 & 31.83 & 31.17 \\
 & $\epsilon$=10 & 11.38 & 11.81 & 28.26 & 26.86 & 10.61 & 10.89 & 26.30 & 25.19 \\
 & $\epsilon$=100 & 11.95 & 12.55 & 30.64 & 29.94 & 10.15 & 10.98 & 22.10 & 19.97 \\
 & $\epsilon$=$\infty$ & 12.16 & 12.86 & 31.96 & 30.08 & 10.35 & 10.90 & 19.89 & 20.26 \\
\cline{1-10}
\multirow[c]{3}{*}{GreedyBayes} & $\rho$=2 & 17.20 & 16.79 & 36.37 & 33.72 & 14.70 & 15.10 & 23.70 & 21.83 \\
 & $\rho$=3 & 13.74 & 13.23 & 29.89 & 30.47 & 11.53 & 10.36 & 19.56 & 18.56 \\
 & $\rho$=4 & 11.09 & 11.16 & 26.47 & 25.62 & 7.98 & 7.36 & 16.83 & 14.07 \\
\cline{1-10}
\multirow[c]{4}{*}{MST} & $\epsilon$=1 & 21.67 & 20.87 & 38.17 & 37.47 & 19.63 & 22.04 & 42.28 & 38.80 \\
 & $\epsilon$=10 & 21.20 & 20.49 & 37.80 & 37.12 & 18.50 & 19.06 & 31.24 & 28.51 \\
 & $\epsilon$=100 & 21.21 & 20.51 & 37.73 & 37.16 & 18.42 & 18.93 & 30.17 & 27.55 \\
 & $\epsilon$=$\infty$ & 21.16 & 20.46 & 37.72 & 37.12 & 18.44 & 18.93 & 30.14 & 27.55 \\
\cline{1-10}
\multirow[c]{3}{*}{PrivBayes} & $\epsilon$=1 & 18.52 & 18.76 & 37.88 & 36.06 & 18.75 & 19.70 & 36.14 & 34.66 \\
 & $\epsilon$=10 & 16.55 & 16.65 & 35.79 & 33.72 & 14.89 & 15.50 & 29.63 & 28.60 \\
 & $\epsilon$=100 & 13.89 & 14.96 & 35.24 & 33.12 & 13.12 & 14.08 & 28.46 & 25.61 \\
\cline{1-10}
\end{tabularx}
\end{table}

%% file: tables/acs_tvd_global.tex
\begin{table}
\centering
\caption{TVD between synthetic and real ACS data, averaged across replicas.}
\label{table:tvd-global-acs}
\begin{tabularx}{0.95\textwidth}{|ll|CCCC|CCCC|}
\cline{1-10} &  & \multicolumn{4}{c|}{population-wide} & \multicolumn{4}{c|}{groupwise} \\
 &  & BF & BM & WF & WM & BF & BM & WF & WM \\
\cline{1-10}\multirow[c]{4}{*}{AIM} & $\epsilon$=1 & 98.75 & 99.00 & 92.63 & 91.59 & 49.52 & 49.65 & 47.08 & 46.70 \\
 & $\epsilon$=10 & 98.34 & 98.76 & 91.49 & 90.58 & 49.31 & 49.48 & 46.47 & 46.19 \\
 & $\epsilon$=100 & 98.07 & 98.61 & 91.05 & 90.23 & 48.82 & 49.15 & 46.26 & 45.79 \\
 & $\epsilon$=$\infty$ & 98.14 & 98.73 & 90.70 & 90.07 & 48.47 & 48.06 & 46.03 & 45.74 \\
\cline{1-10}
\multirow[c]{3}{*}{GreedyBayes} & $\rho$=2 & 98.89 & 99.25 & 92.86 & 92.06 & 48.60 & 48.87 & 46.93 & 46.57 \\
 & $\rho$=3 & 96.43 & 96.63 & 88.57 & 87.83 & 45.81 & 44.60 & 43.89 & 44.91 \\
 & $\rho$=4 & 88.57 & 89.58 & 79.01 & 78.65 & 39.06 & 38.75 & 39.40 & 39.53 \\
\cline{1-10}
\multirow[c]{4}{*}{MST} & $\epsilon$=1 & 99.06 & 99.33 & 94.27 & 94.03 & 49.62 & 49.80 & 47.68 & 47.22 \\
 & $\epsilon$=10 & 99.08 & 99.31 & 94.23 & 93.89 & 49.40 & 49.59 & 47.46 & 47.13 \\
 & $\epsilon$=100 & 99.03 & 99.32 & 94.22 & 93.93 & 49.36 & 49.59 & 47.45 & 47.11 \\
 & $\epsilon$=$\infty$ & 49.57 & 49.67 & 47.55 & 47.37 & 49.36 & 49.59 & 47.46 & 47.12 \\
\cline{1-10}
\multirow[c]{3}{*}{PrivBayes} & $\epsilon$=1 & 49.53 & 49.63 & 47.27 & 47.00 & 49.61 & 49.76 & 47.43 & 46.98 \\
 & $\epsilon$=10 & 49.40 & 49.55 & 46.96 & 46.64 & 49.38 & 49.54 & 46.88 & 46.63 \\
 & $\epsilon$=100 & 49.27 & 49.46 & 46.69 & 46.21 & 49.12 & 49.38 & 46.48 & 46.15 \\
\cline{1-10}
\end{tabularx}
\end{table}

%% file: tables/acs_auroc_diff.tex
\begin{table}
\centering
\caption{Difference of AUROC when training from synthetic rather than real ACS data, evaluated on a real-world testing set and averaged across replicas.}
\label{table:auroc-diff-acs}
\begin{tabular}{|ll|ccccc|ccccc|}
\cline{1-12} &  & \multicolumn{5}{c|}{population-wide} & \multicolumn{5}{c|}{groupwise} \\
 &  & all & WM & WF & BM & BF & all & WM & WF & BM & BF \\
\cline{1-12}\multirow[c]{4}{*}{AIM} & $\epsilon$=1 & -1.34 & -1.20 & -1.56 & -2.04 & -1.91 & -1.47 & -1.36 & -1.56 & -2.55 & -2.71 \\
 & $\epsilon$=10 & -0.90 & -0.77 & -1.05 & -1.37 & -1.27 & -0.68 & -0.62 & -0.77 & -1.21 & -0.75 \\
 & $\epsilon$=100 & -1.00 & -0.90 & -1.09 & -1.38 & -1.40 & -0.56 & -0.55 & -0.63 & -0.91 & -0.49 \\
 & $\epsilon$=$\infty$ & -0.91 & -0.74 & -1.08 & -1.34 & -1.48 & -0.88 & -0.68 & -1.18 & -1.41 & -0.82 \\
\cline{1-12}
\multirow[c]{3}{*}{GreedyBayes} & $\rho$=2 & -3.69 & -3.85 & -3.81 & -4.77 & -4.66 & -3.57 & -3.70 & -3.56 & -6.33 & -5.50 \\
 & $\rho$=3 & -4.86 & -5.42 & -3.89 & -6.14 & -4.55 & -3.04 & -2.74 & -3.35 & -8.26 & -4.01 \\
 & $\rho$=4 & -2.91 & -3.06 & -2.93 & -4.02 & -3.52 & -1.18 & -1.37 & -1.00 & -1.94 & -2.02 \\
\cline{1-12}
\multirow[c]{4}{*}{MST} & $\epsilon$=1 & -9.39 & -7.93 & -10.35 & -8.03 & -9.95 & -9.05 & -7.33 & -12.79 & -16.67 & -11.92 \\
 & $\epsilon$=10 & -9.39 & -7.89 & -10.37 & -8.05 & -9.97 & -8.97 & -7.30 & -12.65 & -15.91 & -11.58 \\
 & $\epsilon$=100 & -9.39 & -7.91 & -10.32 & -7.98 & -9.93 & -9.01 & -7.30 & -12.78 & -16.28 & -11.62 \\
 & $\epsilon$=$\infty$ & -9.40 & -7.94 & -10.36 & -8.02 & -9.93 & -8.97 & -7.29 & -12.69 & -16.17 & -11.43 \\
\cline{1-12}
\multirow[c]{3}{*}{PrivBayes} & $\epsilon$=1 & -2.30 & -2.23 & -2.44 & -2.50 & -2.31 & -2.33 & -2.21 & -2.46 & -3.84 & -3.28 \\
 & $\epsilon$=10 & -1.44 & -1.36 & -1.60 & -2.41 & -1.80 & -1.25 & -1.23 & -1.26 & -2.50 & -1.76 \\
 & $\epsilon$=100 & -1.67 & -1.53 & -1.89 & -2.18 & -2.28 & -1.07 & -0.85 & -1.12 & -3.41 & -2.82 \\
\cline{1-12}
\end{tabular}
\end{table}

%% file: tables/acs_aurpc_diff.tex
\begin{table}
\centering
\caption{Difference of AURPC when training from synthetic rather than real ACS data, evaluated on a real-world testing set and averaged across replicas.}
\label{table:aurpc-diff-acs}
\begin{tabular}{|ll|ccccc|ccccc|}
\cline{1-12} &  & \multicolumn{5}{c|}{population-wide} & \multicolumn{5}{c|}{groupwise} \\
 &  & all & WM & WF & BM & BF & all & WM & WF & BM & BF \\
\cline{1-12}\multirow[c]{4}{*}{AIM} & $\epsilon$=1 & -1.94 & -1.48 & -2.75 & -2.91 & -2.92 & -2.19 & -1.70 & -2.80 & -3.75 & -3.70 \\
 & $\epsilon$=10 & -1.20 & -0.86 & -1.66 & -1.95 & -1.75 & -0.90 & -0.69 & -1.24 & -1.67 & -0.80 \\
 & $\epsilon$=100 & -1.49 & -1.11 & -1.90 & -2.08 & -1.98 & -0.77 & -0.64 & -0.97 & -1.29 & -0.52 \\
 & $\epsilon$=$\infty$ & -1.26 & -0.86 & -1.78 & -1.93 & -2.11 & -1.26 & -0.82 & -2.14 & -2.73 & -1.45 \\
\cline{1-12}
\multirow[c]{3}{*}{GreedyBayes} & $\rho$=2 & -5.75 & -4.94 & -7.11 & -7.66 & -7.41 & -5.58 & -4.74 & -6.95 & -10.86 & -7.90 \\
 & $\rho$=3 & -6.69 & -5.78 & -6.57 & -8.73 & -7.01 & -4.35 & -3.35 & -5.53 & -17.23 & -6.40 \\
 & $\rho$=4 & -4.21 & -3.53 & -5.21 & -6.40 & -5.74 & -1.95 & -1.76 & -2.16 & -3.52 & -3.36 \\
\cline{1-12}
\multirow[c]{4}{*}{MST} & $\epsilon$=1 & -15.17 & -10.88 & -19.66 & -14.33 & -19.12 & -13.39 & -10.30 & -22.96 & -28.86 & -22.75 \\
 & $\epsilon$=10 & -15.23 & -10.87 & -19.86 & -14.76 & -19.29 & -13.35 & -10.30 & -22.82 & -28.37 & -22.45 \\
 & $\epsilon$=100 & -15.14 & -10.84 & -19.57 & -14.50 & -18.94 & -13.37 & -10.29 & -23.01 & -28.12 & -22.08 \\
 & $\epsilon$=$\infty$ & -15.17 & -10.91 & -19.74 & -14.71 & -18.93 & -13.33 & -10.30 & -22.74 & -28.04 & -22.33 \\
\cline{1-12}
\multirow[c]{3}{*}{PrivBayes} & $\epsilon$=1 & -3.16 & -2.45 & -4.06 & -3.56 & -3.45 & -3.32 & -2.52 & -4.32 & -6.44 & -5.70 \\
 & $\epsilon$=10 & -2.19 & -1.74 & -2.75 & -3.93 & -3.14 & -1.76 & -1.44 & -2.05 & -4.10 & -2.64 \\
 & $\epsilon$=100 & -2.54 & -1.94 & -3.50 & -3.53 & -3.77 & -1.41 & -0.99 & -1.87 & -5.43 & -3.53 \\
\cline{1-12}
\end{tabular}
\end{table}

%% file: tables/acs_bal_acc_diff.tex
\begin{table}
\centering
\caption{Difference of Balanced Accuracy when training from synthetic rather than real ACS data, evaluated on a real-world testing set and averaged across replicas.}
\label{table:bal_acc-diff-acs}
\begin{tabular}{|ll|ccccc|ccccc|}
\cline{1-12} &  & \multicolumn{5}{c|}{population-wide} & \multicolumn{5}{c|}{groupwise} \\
 &  & all & WM & WF & BM & BF & all & WM & WF & BM & BF \\
\cline{1-12}\multirow[c]{4}{*}{AIM} & $\epsilon$=1 & -1.21 & -1.33 & -1.18 & -1.25 & -0.77 & -1.30 & -1.40 & -1.21 & -1.71 & -0.75 \\
 & $\epsilon$=10 & -0.79 & -0.65 & -0.94 & -0.65 & -0.41 & -0.68 & -0.82 & -0.59 & -0.56 & 0.22 \\
 & $\epsilon$=100 & -0.93 & -1.01 & -1.05 & -0.76 & -0.98 & -0.52 & -0.66 & -0.44 & -0.66 & 0.06 \\
 & $\epsilon$=$\infty$ & -0.83 & -0.74 & -0.95 & -0.45 & -0.64 & -0.73 & -0.65 & -0.97 & -1.15 & -0.34 \\
\cline{1-12}
\multirow[c]{3}{*}{GreedyBayes} & $\rho$=2 & -3.37 & -3.49 & -3.50 & -3.59 & -3.85 & -3.28 & -3.41 & -3.27 & -5.82 & -5.15 \\
 & $\rho$=3 & -4.43 & -4.47 & -3.16 & -4.87 & -2.72 & -2.82 & -3.13 & -3.11 & -6.50 & -4.52 \\
 & $\rho$=4 & -2.81 & -2.74 & -2.74 & -3.00 & -3.19 & -1.14 & -1.47 & -0.94 & -1.55 & -2.13 \\
\cline{1-12}
\multirow[c]{4}{*}{MST} & $\epsilon$=1 & -8.53 & -7.18 & -9.01 & -6.10 & -7.32 & -8.06 & -7.38 & -10.96 & -12.71 & -9.81 \\
 & $\epsilon$=10 & -8.52 & -7.04 & -8.94 & -6.37 & -7.52 & -7.95 & -7.36 & -10.94 & -11.34 & -9.63 \\
 & $\epsilon$=100 & -8.56 & -7.14 & -9.04 & -6.14 & -7.40 & -8.05 & -7.40 & -11.11 & -12.39 & -9.65 \\
 & $\epsilon$=$\infty$ & -8.53 & -7.15 & -9.03 & -6.06 & -7.55 & -8.01 & -7.38 & -11.02 & -12.92 & -9.66 \\
\cline{1-12}
\multirow[c]{3}{*}{PrivBayes} & $\epsilon$=1 & -2.29 & -3.05 & -2.36 & -1.42 & -1.05 & -2.22 & -2.90 & -2.19 & -2.79 & -2.31 \\
 & $\epsilon$=10 & -1.25 & -1.19 & -1.44 & -1.67 & -1.11 & -1.11 & -1.22 & -1.04 & -1.49 & -1.17 \\
 & $\epsilon$=100 & -1.52 & -1.32 & -1.81 & -1.62 & -1.18 & -0.94 & -0.75 & -0.88 & -3.23 & -2.54 \\
\cline{1-12}
\end{tabular}
\end{table}

%% file: tables/acs_accuracy_diff.tex
\begin{table}
\centering
\caption{Difference of Accuracy when training from synthetic rather than real ACS data, evaluated on a real-world testing set and averaged across replicas.}
\label{table:accuracy-diff-acs}
\begin{tabular}{|ll|ccccc|ccccc|}
\cline{1-12} &  & \multicolumn{5}{c|}{population-wide} & \multicolumn{5}{c|}{groupwise} \\
 &  & all & WM & WF & BM & BF & all & WM & WF & BM & BF \\
\cline{1-12}\multirow[c]{4}{*}{AIM} & $\epsilon$=1 & -1.52 & -1.40 & -1.58 & -2.24 & -1.49 & -1.58 & -1.46 & -1.53 & -2.34 & -2.33 \\
 & $\epsilon$=10 & -0.80 & -0.65 & -0.93 & -1.29 & -0.75 & -0.86 & -0.86 & -0.78 & -1.38 & -1.08 \\
 & $\epsilon$=100 & -0.99 & -1.05 & -0.80 & -2.26 & -0.92 & -0.65 & -0.69 & -0.61 & -0.84 & -0.50 \\
 & $\epsilon$=$\infty$ & -0.80 & -0.73 & -0.81 & -1.60 & -0.73 & -0.77 & -0.67 & -0.86 & -0.81 & -0.92 \\
\cline{1-12}
\multirow[c]{3}{*}{GreedyBayes} & $\rho$=2 & -3.26 & -3.51 & -2.74 & -5.49 & -3.23 & -3.20 & -3.45 & -2.95 & -4.49 & -1.84 \\
 & $\rho$=3 & -4.19 & -4.29 & -3.90 & -5.13 & -4.73 & -3.00 & -3.28 & -2.37 & -7.77 & -1.41 \\
 & $\rho$=4 & -2.63 & -2.69 & -2.44 & -3.93 & -2.43 & -1.17 & -1.51 & -0.87 & -1.13 & -0.64 \\
\cline{1-12}
\multirow[c]{4}{*}{MST} & $\epsilon$=1 & -9.66 & -7.24 & -12.29 & -7.88 & -11.88 & -8.91 & -7.81 & -9.74 & -15.43 & -6.51 \\
 & $\epsilon$=10 & -9.50 & -7.05 & -12.14 & -7.74 & -11.91 & -8.76 & -7.79 & -9.39 & -15.41 & -6.58 \\
 & $\epsilon$=100 & -9.60 & -7.17 & -12.26 & -7.67 & -11.84 & -8.84 & -7.84 & -9.52 & -15.63 & -6.46 \\
 & $\epsilon$=$\infty$ & -9.65 & -7.20 & -12.30 & -7.80 & -12.01 & -8.78 & -7.81 & -9.44 & -15.44 & -6.32 \\
\cline{1-12}
\multirow[c]{3}{*}{PrivBayes} & $\epsilon$=1 & -2.66 & -3.22 & -1.85 & -5.01 & -2.15 & -2.59 & -3.07 & -1.91 & -3.53 & -2.88 \\
 & $\epsilon$=10 & -1.29 & -1.21 & -1.22 & -2.87 & -1.14 & -1.27 & -1.27 & -1.17 & -2.43 & -1.09 \\
 & $\epsilon$=100 & -1.39 & -1.30 & -1.36 & -2.09 & -1.95 & -0.92 & -0.75 & -1.03 & -1.71 & -0.83 \\
\cline{1-12}
\end{tabular}
\end{table}

%% file: tables/acs_auroc.tex
\begin{table}
\centering
\caption{AUROC of classifiers trained from synthetic ACS data, evaluated on a real-world testing set and averaged across replicas.}
\label{table:auroc-acs}
\begin{tabularx}{0.95\textwidth}{|ll|CCCCC|CCCCC|}
\cline{1-12} &  & \multicolumn{5}{c|}{population-wide} & \multicolumn{5}{c|}{groupwise} \\
 &  & all & WM & WF & BM & BF & all & WM & WF & BM & BF \\
\cline{1-12}\multirow[c]{4}{*}{AIM} & $\epsilon$=1 & 87.67 & 86.96 & 87.55 & 85.18 & 85.88 & 87.54 & 86.80 & 87.55 & 84.67 & 85.08 \\
 & $\epsilon$=10 & 88.11 & 87.39 & 88.06 & 85.85 & 86.52 & 88.33 & 87.54 & 88.33 & 86.01 & 87.03 \\
 & $\epsilon$=100 & 88.00 & 87.26 & 88.02 & 85.84 & 86.39 & 88.45 & 87.61 & 88.47 & 86.31 & 87.30 \\
 & $\epsilon$=$\infty$ & 88.09 & 87.42 & 88.03 & 85.88 & 86.30 & 88.13 & 87.48 & 87.92 & 85.81 & 86.96 \\
\cline{1-12}
\multirow[c]{3}{*}{GreedyBayes} & $\rho$=2 & 85.32 & 84.31 & 85.30 & 82.46 & 83.13 & 85.44 & 84.46 & 85.55 & 80.89 & 82.29 \\
 & $\rho$=3 & 84.14 & 82.74 & 85.22 & 81.08 & 83.24 & 85.97 & 85.42 & 85.75 & 78.96 & 83.78 \\
 & $\rho$=4 & 86.09 & 85.10 & 86.18 & 83.20 & 84.27 & 87.83 & 86.80 & 88.11 & 85.28 & 85.77 \\
\cline{1-12}
\multirow[c]{4}{*}{MST} & $\epsilon$=1 & 79.62 & 80.23 & 78.76 & 79.19 & 77.84 & 79.96 & 80.83 & 76.32 & 70.55 & 75.87 \\
 & $\epsilon$=10 & 79.62 & 80.27 & 78.74 & 79.17 & 77.82 & 80.03 & 80.86 & 76.46 & 71.31 & 76.21 \\
 & $\epsilon$=100 & 79.62 & 80.25 & 78.79 & 79.24 & 77.86 & 80.00 & 80.87 & 76.33 & 70.94 & 76.17 \\
 & $\epsilon$=$\infty$ & 79.61 & 80.22 & 78.74 & 79.20 & 77.86 & 80.03 & 80.87 & 76.42 & 71.05 & 76.36 \\
\cline{1-12}
\multirow[c]{3}{*}{PrivBayes} & $\epsilon$=1 & 86.70 & 85.93 & 86.67 & 84.72 & 85.48 & 86.67 & 85.95 & 86.65 & 83.38 & 84.51 \\
 & $\epsilon$=10 & 87.56 & 86.80 & 87.51 & 84.82 & 85.99 & 87.76 & 86.93 & 87.84 & 84.72 & 86.03 \\
 & $\epsilon$=100 & 87.34 & 86.63 & 87.22 & 85.04 & 85.51 & 87.94 & 87.31 & 87.99 & 83.81 & 84.96 \\
\cline{1-12}
\end{tabularx}
\end{table}

%% file: tables/acs_aurpc.tex
\begin{table}
\centering
\caption{AURPC of classifiers trained from synthetic ACS data, evaluated on a real-world testing set and averaged across replicas.}
\label{table:aurpc-acs}
\begin{tabularx}{0.95\textwidth}{|ll|CCCCC|CCCCC|}
\cline{1-12} &  & \multicolumn{5}{c|}{population-wide} & \multicolumn{5}{c|}{groupwise} \\
 &  & all & WM & WF & BM & BF & all & WM & WF & BM & BF \\
\cline{1-12}\multirow[c]{4}{*}{AIM} & $\epsilon$=1 & 81.39 & 85.05 & 76.22 & 70.92 & 67.23 & 81.14 & 84.83 & 76.17 & 70.08 & 66.45 \\
 & $\epsilon$=10 & 82.13 & 85.66 & 77.31 & 71.88 & 68.40 & 82.43 & 85.84 & 77.74 & 72.16 & 69.35 \\
 & $\epsilon$=100 & 81.84 & 85.42 & 77.08 & 71.75 & 68.17 & 82.56 & 85.89 & 78.01 & 72.54 & 69.62 \\
 & $\epsilon$=$\infty$ & 82.07 & 85.67 & 77.20 & 71.90 & 68.04 & 82.07 & 85.70 & 76.83 & 71.10 & 68.70 \\
\cline{1-12}
\multirow[c]{3}{*}{GreedyBayes} & $\rho$=2 & 77.58 & 81.59 & 71.86 & 66.17 & 62.74 & 77.75 & 81.79 & 72.02 & 62.97 & 62.25 \\
 & $\rho$=3 & 76.64 & 80.75 & 72.40 & 65.10 & 63.14 & 78.98 & 83.18 & 73.44 & 56.61 & 63.74 \\
 & $\rho$=4 & 79.12 & 83.00 & 73.77 & 67.43 & 64.41 & 81.39 & 84.77 & 76.81 & 70.31 & 66.79 \\
\cline{1-12}
\multirow[c]{4}{*}{MST} & $\epsilon$=1 & 68.16 & 75.65 & 59.31 & 59.50 & 51.03 & 69.95 & 76.23 & 56.01 & 44.97 & 47.40 \\
 & $\epsilon$=10 & 68.10 & 75.66 & 59.11 & 59.07 & 50.86 & 69.98 & 76.23 & 56.15 & 45.47 & 47.70 \\
 & $\epsilon$=100 & 68.19 & 75.69 & 59.40 & 59.34 & 51.20 & 69.97 & 76.24 & 55.96 & 45.71 & 48.06 \\
 & $\epsilon$=$\infty$ & 68.17 & 75.62 & 59.23 & 59.12 & 51.22 & 70.00 & 76.23 & 56.24 & 45.79 & 47.82 \\
\cline{1-12}
\multirow[c]{3}{*}{PrivBayes} & $\epsilon$=1 & 80.17 & 84.08 & 74.91 & 70.27 & 66.70 & 80.01 & 84.01 & 74.65 & 67.39 & 64.44 \\
 & $\epsilon$=10 & 81.14 & 84.79 & 76.22 & 69.90 & 67.01 & 81.57 & 85.08 & 76.93 & 69.73 & 67.51 \\
 & $\epsilon$=100 & 80.80 & 84.58 & 75.47 & 70.30 & 66.38 & 81.92 & 85.54 & 77.10 & 68.40 & 66.62 \\
\cline{1-12}
\end{tabularx}
\end{table}

%% file: tables/acs_bal_acc.tex
\begin{table}
\centering
\caption{Balanced Accuracy of classifiers trained from synthetic ACS data, evaluated on a real-world testing set and averaged across replicas.}
\label{table:bal_acc-acs}
\begin{tabularx}{0.95\textwidth}{|ll|CCCCC|CCCCC|}
\cline{1-12} &  & \multicolumn{5}{c|}{population-wide} & \multicolumn{5}{c|}{groupwise} \\
 &  & all & WM & WF & BM & BF & all & WM & WF & BM & BF \\
\cline{1-12}\multirow[c]{4}{*}{AIM} & $\epsilon$=1 & 79.36 & 77.72 & 79.27 & 76.36 & 76.89 & 79.27 & 77.65 & 79.24 & 75.90 & 76.92 \\
 & $\epsilon$=10 & 79.77 & 78.40 & 79.52 & 76.96 & 77.26 & 79.89 & 78.23 & 79.87 & 77.05 & 77.89 \\
 & $\epsilon$=100 & 79.64 & 78.04 & 79.40 & 76.85 & 76.68 & 80.05 & 78.40 & 80.02 & 76.95 & 77.72 \\
 & $\epsilon$=$\infty$ & 79.73 & 78.31 & 79.50 & 77.16 & 77.03 & 79.84 & 78.41 & 79.49 & 76.46 & 77.33 \\
\cline{1-12}
\multirow[c]{3}{*}{GreedyBayes} & $\rho$=2 & 77.20 & 75.56 & 76.95 & 74.02 & 73.81 & 77.29 & 75.64 & 77.18 & 71.79 & 72.52 \\
 & $\rho$=3 & 76.14 & 74.58 & 77.30 & 72.74 & 74.95 & 77.74 & 75.92 & 77.35 & 71.11 & 73.15 \\
 & $\rho$=4 & 77.76 & 76.31 & 77.71 & 74.61 & 74.48 & 79.43 & 77.58 & 79.52 & 76.06 & 75.53 \\
\cline{1-12}
\multirow[c]{4}{*}{MST} & $\epsilon$=1 & 72.04 & 71.87 & 71.44 & 71.51 & 70.35 & 72.51 & 71.67 & 69.49 & 64.90 & 67.86 \\
 & $\epsilon$=10 & 72.05 & 72.01 & 71.52 & 71.25 & 70.15 & 72.61 & 71.69 & 69.51 & 66.27 & 68.03 \\
 & $\epsilon$=100 & 72.00 & 71.91 & 71.41 & 71.47 & 70.26 & 72.52 & 71.65 & 69.35 & 65.22 & 68.02 \\
 & $\epsilon$=$\infty$ & 72.04 & 71.91 & 71.42 & 71.55 & 70.12 & 72.55 & 71.68 & 69.44 & 64.69 & 68.00 \\
\cline{1-12}
\multirow[c]{3}{*}{PrivBayes} & $\epsilon$=1 & 78.28 & 76.00 & 78.10 & 76.19 & 76.61 & 78.35 & 76.15 & 78.27 & 74.82 & 75.36 \\
 & $\epsilon$=10 & 79.32 & 77.86 & 79.02 & 75.94 & 76.55 & 79.46 & 77.83 & 79.42 & 76.12 & 76.49 \\
 & $\epsilon$=100 & 79.05 & 77.74 & 78.65 & 75.99 & 76.48 & 79.63 & 78.30 & 79.57 & 74.38 & 75.12 \\
\cline{1-12}
\end{tabularx}
\end{table}

%% file: tables/acs_accuracy.tex
\begin{table}
\centering
\caption{Accuracy of classifiers trained from synthetic ACS data, evaluated on a real-world testing set and averaged across replicas.}
\label{table:accuracy-acs}
\begin{tabularx}{0.95\textwidth}{|ll|CCCCC|CCCCC|}
\cline{1-12} &  & \multicolumn{5}{c|}{population-wide} & \multicolumn{5}{c|}{groupwise} \\
 &  & all & WM & WF & BM & BF & all & WM & WF & BM & BF \\
\cline{1-12}\multirow[c]{4}{*}{AIM} & $\epsilon$=1 & 78.56 & 77.31 & 79.77 & 77.36 & 81.18 & 78.50 & 77.25 & 79.83 & 77.26 & 80.34 \\
 & $\epsilon$=10 & 79.28 & 78.06 & 80.42 & 78.31 & 81.92 & 79.22 & 77.85 & 80.57 & 78.22 & 81.59 \\
 & $\epsilon$=100 & 79.09 & 77.66 & 80.55 & 77.34 & 81.75 & 79.43 & 78.02 & 80.74 & 78.76 & 82.17 \\
 & $\epsilon$=$\infty$ & 79.28 & 77.98 & 80.55 & 78.00 & 81.94 & 79.31 & 78.04 & 80.50 & 78.80 & 81.75 \\
\cline{1-12}
\multirow[c]{3}{*}{GreedyBayes} & $\rho$=2 & 76.82 & 75.20 & 78.61 & 74.12 & 79.44 & 76.88 & 75.26 & 78.40 & 75.11 & 80.83 \\
 & $\rho$=3 & 75.89 & 74.42 & 77.45 & 74.47 & 77.94 & 77.08 & 75.43 & 78.98 & 71.83 & 81.26 \\
 & $\rho$=4 & 77.45 & 76.01 & 78.91 & 75.67 & 80.24 & 78.91 & 77.20 & 80.49 & 78.47 & 82.03 \\
\cline{1-12}
\multirow[c]{4}{*}{MST} & $\epsilon$=1 & 70.42 & 71.47 & 69.07 & 71.72 & 70.79 & 71.17 & 70.90 & 71.61 & 64.17 & 76.16 \\
 & $\epsilon$=10 & 70.58 & 71.66 & 69.21 & 71.87 & 70.76 & 71.32 & 70.92 & 71.96 & 64.19 & 76.09 \\
 & $\epsilon$=100 & 70.48 & 71.54 & 69.09 & 71.93 & 70.83 & 71.24 & 70.87 & 71.83 & 63.97 & 76.21 \\
 & $\epsilon$=$\infty$ & 70.43 & 71.51 & 69.05 & 71.80 & 70.66 & 71.30 & 70.90 & 71.91 & 64.16 & 76.35 \\
\cline{1-12}
\multirow[c]{3}{*}{PrivBayes} & $\epsilon$=1 & 77.42 & 75.49 & 79.50 & 74.59 & 80.52 & 77.49 & 75.64 & 79.44 & 76.07 & 79.79 \\
 & $\epsilon$=10 & 78.79 & 77.49 & 80.14 & 76.73 & 81.53 & 78.81 & 77.44 & 80.19 & 77.17 & 81.58 \\
 & $\epsilon$=100 & 78.69 & 77.41 & 80.00 & 77.52 & 80.72 & 79.16 & 77.96 & 80.33 & 77.89 & 81.84 \\
\cline{1-12}
\end{tabularx}
\end{table}